\documentclass[11pt]{article}
\usepackage[margin=1in]{geometry}

\usepackage[utf8]{inputenc} % allow utf-8 input
\usepackage[T1]{fontenc}    % use 8-bit T1 fonts
\usepackage{hyperref}       % hyperlinks
\usepackage{url}            % simple URL typesetting
\usepackage{booktabs}       % professional-quality tables
\usepackage{amsfonts}       % blackboard math symbols
\usepackage{nicefrac}       % compact symbols for 1/2, etc.
\usepackage{microtype}      % microtypography
\usepackage{multicol}

\usepackage{authblk}
\newcommand*\samethanks[1][\value{footnote}]{\footnotemark[#1]}

\usepackage{xr}
\makeatletter
\newcommand*{\addFileDependency}[1]{% argument=file name and extension
  \typeout{(#1)}
  \@addtofilelist{#1}
  \IfFileExists{#1}{}{\typeout{No file #1.}}
}
\makeatother

\usepackage{wrapfig}

\usepackage[ruled, vlined, linesnumbered]{algorithm2e}
\usepackage{amsmath,amssymb,algorithmicx,algpseudocode,amsthm,bm}
\usepackage{thm-restate}
\usepackage{graphicx}

\newtheorem{theorem}{Theorem}

\newtheorem{lemma}{Lemma}[section]
\newtheorem{fact}{Fact}

\algnewcommand{\LeftComment}[1]{\Statex // \emph{#1}}
\makeatletter
\algnewcommand{\LineComment}[1]{\Statex \hskip\ALG@thistlm \(\) // \emph{#1}}
\makeatother

\title{Distributed Bandit Learning: Near-Optimal Regret with Efficient Communication}

\author[1]{Yuanhao Wang\thanks{Equal contribution}}
\author[2]{Jiachen Hu\samethanks}
\author[3]{Xiaoyu Chen}
\author[3, 4]{Liwei Wang}
\affil[1]{Institute for Interdisciplinary Information Sciences, Tsinghua University}
\affil[2]{School of Electronics Engineering and Computer Science,
Peking University}
\affil[3]{Key Laboratory of Machine Perception, MOE, School of EECS, Peking University}
\affil[4]{Center for Data Science, Peking University, Beijing Institute of Big Data Research}

\affil[ ]{\texttt{yuanhao-16@mails.tsinghua.edu.cn\\ \{NickH, cxy30\}@pku.edu.cn\\wanglw@cis.pku.edu.cn}}

\begin{document}

\maketitle

\begin{abstract}
%We study the regret minimization of distributed multi-armed bandits (MAB) and distributed linear bandits, in which $ M $ agents work collaboratively to minimize total regret under the coordinate of a central server using as little communication as possible. For distributed $K$-armed bandits, we propose an elimination-based protocol that communicates only $O(M\log(MK))$ numbers and achieves near-optimal regret. This establishes that there is almost no trade off between communication efficiency and regret for distributed MAB. For distributed $d$-dimensional linear bandits, we propose a elimination-based protocol with near-optimal regret and $\tilde{O}(Md)$ communicated numbers.

We study the problem of regret minimization for distributed bandits learning, in which $M$ agents work collaboratively to minimize their total regret under the coordination of a central server. Our goal is to design communication protocols with near-optimal regret and little communication cost, which is measured by the total amount of transmitted data. For distributed multi-armed bandits, we propose a protocol with near-optimal regret and only $O(M\log(MK))$ communication cost, where $K$ is the number of arms. The communication cost is independent of the time horizon $T$, has only logarithmic dependence on the number of arms, and matches the lower bound except for a logarithmic factor. %and grows only logarithmically in the number of arms.
%We prove the communication lower bounds in order to achieve near-optimal regret and the communication cost of DEMAB matches the lower bound except for a logarithmic factor.
%For distributed $d$-dimensional linear bandits, we propose two protocols for different bandit settings that achieve near-optimal regret. They have communication cost of order $\tilde{O}(Md)$ (fixed action set) and $O(M^{1.5}d^{3})$ (time-varying action set) respectively.
For distributed $d$-dimensional linear bandits, we propose a protocol that achieves near-optimal regret and has communication cost of order $\tilde{O}(Md)$, which has only logarithmic dependence on $T$.

\end{abstract}

%\blfootnote{These two authors contributed equally}

\section{Introduction}
\label{Introduction}
Bandit learning is a central topic in online learning, and has various real-world applications, including clinical trials~\cite{wang1991sequential}, model selection~\cite{maron1994hoeffding} and recommendation systems~\cite{agarwal2009online, li2010contextual,abe2003reinforcement}. In many tasks using bandit algorithms, it is appealing to employ more agents to learn collaboratively and concurrently in order to speed up the learning process. In many other tasks, the sequential decision making is distributed by nature. For instance, multiple spatially separated labs may be working on the same clinical trial. %For instance, multiple spatially separated labs work collaboratively on the same clinical trial.
In such distributed applications, communication between agents is critical, but may also be expensive or time-consuming. This motivates us to consider efficient protocols for distributed learning in bandit problems.

A straightforward communication protocol for bandit learning is \emph{immediate sharing}: each agent shares every new sample immediately with others. Under this scheme, agents can have good collaborative behaviors close to that in a centralized setting. However, the amount of communicated data is directly proportional to the total size of collected samples. When the bandit is played for a long timescale, the cost of communication would render this scheme impractical. A natural question to ask is: How much communication is actually needed for near-optimal performance? In this work, we show that the answer is somewhat surprising: The required communication cost has almost \emph{no} dependence on the time horizon.%On the other hand, if agents do not communicate at all, it is impossible to learn coordinately and achieve better performance. A natural question to ask is: how much communication is needed in order to have near-optimal performance? The answer is surprisingly small, and has almost \emph{no} dependence on timescale.

In this paper, we consider the distributed learning of stochastic multi-armed bandits (MAB) and stochastic linear bandits. There are $M$ agents interacting with the same bandit instance in a synchronized fashion. In time steps $t=1,\cdots,T$, each agent pulls an arm and observes the associated reward. Between time steps, agents can communicate via a server-agent network. Following the typical formulation of single-agent bandit learning, we consider the task of regret minimization \cite{lai1987adaptive, dani2008stochastic,bubeck2012regret}. The total regret of all agents is used as the performance criterion of a communication protocol. The communication cost is measured by the total amount of data communicated in the network. Our goal is to minimize communication cost while maintaining near-optimal performance, that is, regret comparable to the optimal regret of a single agent in $MT$ interactions with the bandit instance.%Our goal is to reduce the amount of communicated data in the network

For multi-armed bandits, we propose the DEMAB protocol, which achieves near-optimal regret. The amount of transmitted data per agent in DEMAB is independent of $T$, and is logarithmic with respect to other parameters. For linear bandits, we propose the DELB protocol, which achieves near-optimal regret, and has communication cost with at most logarithmic dependence on $T$. %when the action set is fixed, and the DisLinUCB protocol has near-optimal regret when the action set is time-varying. Both protocols have communication cost with at most logarithmic dependence on $T$.

%Our results can be briefly summarized as follows. For stochastic multi-armed bandit, we propose a protocol called DEMAB with communication cost $O(M\log(MK))$ and near-optimal regret. The result is quite surprising since the communication cost is independent of time steps $T$. It matches our communication lower bound in order to achieve near-optimal regret except for a logarithmic factor. For linear bandits, we propose two protocols called DELB and DisLinUCB for different linear bandit setting. The communication cost of DELB enjoys nearly linear dependence on both $M$ and $d$, while the communication cost of DisLinUCB has no dependence on time steps $T$.

\subsection{Problem Setting}

\paragraph{Communication Model} The communication network we consider consists of a server and several agents. Agents can communicate with the server by sending or receiving packets. Each data packet contains an integer or a real number. We define the communication cost of a protocol as the number of integers or real numbers communicated between server and agents\footnote{In our protocols, the number of bits each integer or real number uses is only logarithmic w.r.t. instance scale. Using the number of bits as the definition of communication complexity instead will only result in an additional logarithmic factor. The number of communicated bits is analyzed in appendix.}. Several previous works consider the total number of communication rounds \cite{hillel2013distributed,tao2019collaborative}, while we are more interested in the total amount of data transmitted among all rounds. %The data packets contains integers or real numbers, such as the statistical variables of history, the index of arms.
We assume that communication between server and agents has zero latency.
%In other words, server and agents can perform arbitrary amount of communication between timesteps. 
Note that protocols in our model can be easily adapted to a network without a server, by designating an agent as the server.   

%This is different from the definition using the total number of communication bits. by at most a logarithmic factor. We will discuss this in the supplementary materials after analyzing the communication complexity of protocols
\paragraph{Distributed Multi-armed Bandits} In distributed multi-armed bandits, there are $M$ agents, labeled $1$,...,$M$. Each agent is given access to the same stochastic $K$-armed bandit instance. Each arm $ k $ in the instance is associated with a reward distribution $ \mathcal{P}_k $. $ \mathcal{P}_k $ is supported on $ [0, 1] $ with mean $ \mu(k) $. Without loss of generality, we assume that arm 1 is the best arm (i.e. $\mu(1)\ge \mu(k)$, $\forall k\in[K]$). At each time step $ t = 1, 2, ..., T $, each agent $ i $ chooses an arm $ a_{t, i} $, and receives reward $r_{t,i}$ independently sampled from $ \mathcal{P}_{a_{t, i}} $. The goal of the agents is to minimize their total regret, which is defined as
$$REG(T)=\sum_{t=1}^T\sum_{i=1}^{M}\left(\mu(1) - \mu(a_{t,i})\right).$$
For single-agent MAB, i.e., $M=1$, the optimal regret bound is $\Theta(\sqrt{KT})$~\cite{audibert2009minimax}. %Because a single agent can simulate any $M$-agent $T$-step protocol in $MT$ time steps, $\Omega(\sqrt{KMT})$ is a lower bound for distributed MAB regardless of the amount of communication. Therefore, we consider $\Tilde{O}\left(\sqrt{MKT}\right)$ to be near-optimal regret for a protocol.

\paragraph{Distributed Linear Bandits} In distributed linear bandits, the agents are given access to the same $d$-dimensional stochastic linear bandits instance. In particular, we assume that at time step $t$, agents are given an action set $\mathcal{D}\subseteq \left\{x \in \mathbb{R}^{d} : \|x\|_{2} \leq 1\right\}$. Agent $i$ chooses action $x_{t,i} \in \mathcal{D} $ and observes reward $y_{t,i}$. We assume that the mean of the reward is decided by an unknown parameter $ \theta^* \in \mathbb{R}^d $: $ y_{t, i} = x_{t, i}^T \theta^* + \eta_{t, i} $, where $\eta_{t, i} \in [-1,1]$ are independent and have zero mean. For simplicity, we assume $\|\theta^*\|_2 \leq 1$. For distributed linear bandits, the cumulative regret is defined as the sum of individual agent's regrets:
$$REG(T)=\sum_{t=1}^T\sum_{i=1}^M\left(\max_{x\in \mathcal{D}}x^T\theta^*-x_{t,i}^T\theta^*\right).$$
Here, we assume that the action set is fixed. A more general setting considers a time-varying action set $\mathcal{D}_t$. In both cases, algorithms with $O(d\sqrt{T}\log T)$ regret have been proposed~\cite{abbasi2011improved}, while a regret lower bound of $\Omega\left(d\sqrt{T}\right)$ is shown in \cite{dani2008stochastic}.

%In the general time-varying action set case, $\mathcal{D}_t$ may be different in each time step. In the fixed action set case, $\mathcal{D}_1=\cdots=\mathcal{D}_T=\mathcal{D}$. In both cases, algorithms with $O(d\sqrt{T}\log T)$ regret has been proposed~\cite{abbasi2011improved}, while a regret lower bound of $\Omega\left(d\sqrt{T}\right)$ is shown in \cite{dani2008stochastic}.
%TODO: Do we need to explain time-varying action set? We also consider a more general linear bandit setting with time-varying action set $\mathcal{D}_t$. That is, players can choose actions from $\mathcal{D}_t$ at time $t$, while regret is defined against the optimal action in $\mathcal{D}_t$:
%$$REG(T)=\sum_{t=1}^T\sum_{i=1}^M\left(\max_{x\in \mathcal{D}_t}x^T\theta^*-x_{t,i}^T\theta^*\right).$$

%COMMENT: personally I (J. Hu) think we should consider time-varying action set here, and restrict $ \mathcal{D}_t $ to be the same when analyzing DELB.

%For single agent linear bandits of both settings, algorithms with $O(d\sqrt{T}\log T)$ regret has been proposed~\cite{abbasi2011improved}, while a regret lower bound of $\Omega\left(d\sqrt{T}\right)$ is shown in \cite{dani2008stochastic}.% Since any protocol can be simulated by a single agent in $MT$ time steps, this implies a $\Omega\left(d\sqrt{MT}\right)$ regret lower bound for a $M$-agent protocol.%Therefore, is to achieve $\Tilde{O}\left(d\sqrt{MT}\right)$

For both distributed multi-armed bandits and distributed linear bandits, our goal is to use as little communication as possible to achieve near-optimal regret. Since any $M$-agent protocol running for $T$ steps can be simulated by a single-agent bandit algorithm running for $MT$ time steps, the regret of any protocol is lower bounded by the optimal regret of a single-agent algorithm running for $MT$ time steps. Therefore, we consider $\Tilde{O}(\sqrt{MKT})$ regret for multi-armed bandits and $\Tilde{O}(d\sqrt{MT})$ regret for linear bandits to be near-optimal.
%For both distributed multi-armed bandits and distributed linear bandits, our goal is to achieve comparable results to the optimal regret of a single agent bandit algorithm running  for $MT$ time steps with as little communication as possible. That is to say, we hope our distributed protocol can achieve  $\Tilde{O}(\sqrt{MKT})$ regret for multi-armed bandits and  $\Tilde{O}(d\sqrt{MT})$ regret for linear bandits respectively. 

We are mainly interested in the case where the time horizon $T$ is the dominant factor (compared to $M$ or $K$). Unless otherwise stated, we assume that $T>\max\{\frac{M\log M}{K}, M, K,2\}$ in the multi-armed bandits case and $T>\max\{M,2\}$ in the linear bandits case.

\subsection{Our Contribution}
Now we give an overview of our results. In both settings, we present communication-efficient protocols that achieve near-optimal regret. Our results are summarized in Table \ref{tab:result}.

Our results are compared with a naive baseline solution called \textit{immediate sharing} in Table \ref{tab:result}: each agent sends the index of the arm he pulled and the corresponding reward he received to every other agent via the server immediately. This protocol can achieve near-optimal regret for both MAB and linear bandits ($\Tilde{O}(\sqrt{MKT})$ and $\Tilde{O}(d\sqrt{MT})$), but comes with high communication cost ($O(M^2T)$ and $O(M^2dT)$). 

\paragraph{Distributed MAB} For distributed multi-armed bandits, we propose DEMAB (Distributed Elimination for MAB) protocol, which achieves near optimal regret ($O\left(\sqrt{MKT \log T}\right)$) with $O(M\log(MK))$ communication cost. The communication cost is independent of the number of time steps $T$ and grows only logarithmically w.r.t. the number of arms. We also prove the following lower bound: When expected communication cost is less than $M/c$ ($c$ is a universal constant), the total regret is trivially $\Omega(M\sqrt{KT})$. That is, in order to achieve near-optimal regret, the communication cost of DEMAB matches the lower bound except for logarithmic factors.

\paragraph{Distributed Linear Bandits} We propose DELB (Distributed Elimination for Linear Bandits), an elimination based protocol for distributed linear bandits which achieves near-optimal regret bound ($O\left(d\sqrt{MT\log T}\right)$) with communication cost $O\left((Md+d\log\log d)\log T\right)$. The communication cost of DELB enjoys nearly linear dependence on both $M$ and $d$, and has at most logarithmic dependence on $T$. For the more general case where the action set is time-varying, the DisLinUCB (Distributed LinUCB) protocol still achieves near-optimal regret, but requires $O\left(M^{1.5}d^3\right)$ communication cost.%For the scenario where we have a time-varying action set $\mathcal{D}_t$, we propose DisLinUCB (Distributed LinUCB), which achieves $O\left(d\sqrt{MT}\log^2T\right)$ regret with $O\left(M^{1.5}d^3\right)$ communication cost. The communication cost of both protocols have no dependence on time horizon $T$ except for logarithmic factors.

%TODO: do we need to mention a $ \Omega(M) $ communication lower bound of linear bandits?
%In this sense, both of our protocols are more communication-efficient than DCB~\cite{korda2016distributed}, whose communication cost is linear in $T$.

\begin{table}
\renewcommand\arraystretch{1.5}
  \centering
  \begin{tabular}{cccc}
    \toprule
    %\multicolumn{2}{c}{Part}                   \\
    %\cmidrule(r){1-2}
   Setting & Algorithm     & Regret     & Communication \\
    \midrule
  {\centering Multi-armed bandits}
&Immediate Sharing&$O\left(\sqrt{MKT\log T}\right)$&$O(M^2T)$\\
&DEMAB (Sec. \ref{sec:protocol_mab}) &$O\left(\sqrt{MKT\log T}\right)$&$O(M\log (MK))$ \\
&Lower bound (Sec. \ref{sec:mablower})&$o\left(M\sqrt{KT}\right)$&$\Omega(M)$\\  
\midrule

  {Linear bandits}
&Immediate Sharing&$O\left(d\sqrt{MT\log T}\right)$&$O(M^2dT)$\\
&DCB~\cite{korda2016distributed}&$O\left(d\sqrt{MT}\log T\right)$&$O\left(Md^2T\right)$\\
&DELB (Sec. \ref{sec:delb})&$O\left(d\sqrt{MT\log T}\right)$&$O\left(\left(Md+d\log \log d\right)\log T\right)$\\
&DisLinUCB (Sec. \ref{sec:linearucb})&$O\left(d\sqrt{MT}\log^2T\right)$&$O\left(M^{1.5}d^{3}\right)$\\

    \bottomrule
  \end{tabular}
 
   \caption{Summary of baseline approaches and our results}
   \label{tab:result}
\end{table}
%\footnotetext{See Theorem~\ref{thm:lowerbound} for precise statement.}
\section{Related Work}
There has been growing interest in bandits problems with multiple players. One line of research considers the challenging problem of multi-armed bandits with collisions~\cite{rosenski2016multi,bistritz2018distributed,kalathil2014decentralized}, in which the reward for an arm is 0 if it is chosen by more than one player. The task is to minimize regret without communication. Their setting is motivated by problems in cognitive radio networks, and is fairly different from ours. %In their setting, regret of playing $M$ bandits separately is the optimal case, while this is the worst case in our setting. %The work in \cite{landgren2016distributed} considers regret minimization in distributed MAB, but does not consider communication efficiency. The role of communication is quite different in these settings.

In \cite{szorenyi2013gossip} and \cite{korda2016distributed}, the authors consider the distributed learning of MAB and linear bandits with restriction on the communication network. In \cite{szorenyi2013gossip}, motivated by fully decentralized applications, the authors consider P2P communication networks, where an agent can communicate with only two other agents at each time step. A gossip-based $\epsilon$-greedy algorithm is proposed for distributed MAB. Their algorithm achieves a speedup linear in $M$ in terms of error rate, but the communication cost is linear in $T$. %in a communication network where an agent can communicate with only two other agents in each time step. %The regret bound of their algorithm scales as $O(T^{2/3})$, while the communication cost is linear in $T$.
The work of \cite{korda2016distributed} uses a gossip protocol for regret minimization in distributed linear bandits. The main difference between their setting and ours is that each agent is only allowed to communicate with one agent at each time step in \cite{korda2016distributed} \footnote{Our algorithms can be modified to meet this restriction with almost no change in performance.}. Their algorithm achieves near-optimal ($O\left(d\sqrt{MT}\log T\right)$) total regret using $O(Md^2T)$ communication cost.

Another setting in literature concerns about distributed pure exploration in multi-armed bandits \cite{hillel2013distributed,tao2019collaborative}, where the communication model is the most similar one to ours. These works use elimination based protocols for collaborative exploration, and establish tight bounds for communication-speedup tradeoff. In the fixed confidence setting, near optimal ($\Tilde{\Omega}(M)$) speedup is achieved with only $O\left(\log\frac{1}{\epsilon}\right)$ communication rounds when identifying an $\epsilon$-optimal arm~\cite{hillel2013distributed}. In the fixed-time setting, near optimal speedup in exploration can be achieved with only $O\left(\log M\right)$ communication rounds~\cite{tao2019collaborative}. %However, optimal speedup in pure exploration does not imply optimal total regret, and our protocols with near-optimal regret cannot directly imply near-optimal speedup in pure exploration either. 
However, their task (speedup in pure exploration) is not directly comparable to ours (i.e. are not reducible to each other).
Moreover, in \cite{hillel2013distributed,tao2019collaborative}, the number of communication rounds is used as the measure of communication, while we use the amount of transmitted data.

%The authors use the number communication rounds as the measure of communication, and establish tight bounds for round-speedup tradeoff.%These works focus on pure exploration, instead of regret minimization. They also use the number of communication rounds as the measure of communication, while we use the actual amount of transmitted data. 

%In many previous works, communication is either not allowed \cite{rosenski2016multi,bistritz2018distributed}, or not limited \cite{landgren2016distributed}. %In some settings~\cite{rosenski2016multi,bistritz2018distributed}, no communication is allowed, while the task is minimizing regret in multi-armed bandits with collision. That is, reward for an arm is 0 if it is chosen by more than one players. Other previous work do not consider communication efficiency .

%\input{related_work.tex}
\section{Main Results for Multi-armed Bandits}
\label{sec:mab}
\SetAlgorithmName{Protocol}

In this section, we first summarize the single-agent elimination algorithm~\cite{auer2010ucb}, and then present our Distributed Elimination for MAB (DEMAB) protocol. The regret and communication efficiency of the protocol is then analyzed in Sec.~\ref{sec:mabregret}. A communication lower bound is presented in Sec.~\ref{sec:mablower}.

\subsection{Elimination Algorithm for Single-agent MAB}

The elimination algorithm \cite{auer2010ucb} is a near-optimal algorithm for single-agent MAB. The agent acts in phases $l=1,2,\cdots$, and maintains a set of active arms $A_l$. Initially, $A_1=[K]$ (all arms). In phase $l$, each arm in $A_l$ is pulled for $\Theta(4^l\log T)$ times; arms with average reward $2^{-l}$ lower than the maximum are then eliminated from $A_l$. %We use \texttt{Eliminate}($A,T'$) to denote the following subroutine: 1. run the elimination algorithm for $T'$ time steps with $A_1=A$; 2. return the set of remaining arms at $T'$.

For each arm $k\in[K]$, define its suboptimality gap to be $\Delta_k:=\mu(1)-\mu(k)$. In the elimination algorithm, with high probability arm $k$ will be eliminated after approximately $l_k=\log_2\Delta_k^{-1}$ phases, in which it is pulled for at most $O\left(\Delta_k^{-2}\log T\right)$ times. It follows that regret is $O\left(\sum_{k: \Delta_k > 0}\Delta_k^{-1} \log T\right)$, which is almost instance-optimal. The regret is $O\left(\sqrt{KT\log T}\right)$ in the worst case.%It follows that total regret is $O\left(\sqrt{KT\log T}\right)$. Moreover, it also achieves optimal instance-dependent regret $ O\left(\sum_{k: \Delta_k > 0}\Delta_k^{-1} \log T\right)$.

%TODO: I vote for removing it.
%The regret analysis of the elimination algorithm follows the subsequent argument. For each arm $k\in[K]$, define its suboptimality gap to be $\Delta_k:=\mu(1)-\mu(k)$.  If $\Delta_k>0$, and arm $k$ has been pulled for $\Omega\left(\Delta_k^{-2}\log T\right)$ times in a phase, it will be eliminated with high probability. Therefore, with high probability, a suboptimal arm could only remain for $l_k=\log_2\Delta_k^{-1}$ phases. During these phases, it is pulled $O\left(\sum_{i=1}^{l_k}4^l\log T\right)=O\left(\Delta_k^{-2}\log T\right)$ times. Therefore, with high probability, total regret is
%$$REG(T) \le O\left(\sum_{k:\Delta_k>\epsilon}\frac{\log T}{\Delta_k^2}\cdot \Delta_k\right)+\epsilon T = O\left(\sqrt{KT\log T}\right),$$
%where $\epsilon:=\sqrt{K\log T/T}$.

\subsection{The DEMAB Protocol}
\label{sec:protocol_mab}
%The basic idea of the DEMAB protocol is to simulate the single agent elimination algorithm using $M$ agents. The protocol comes in two stages.

\begin{algorithm}[th]
    \label{alg:DEMAB}
    \caption{Distributed Elimination for Multi-Armed Bandits (DEMAB)}
    \DontPrintSemicolon
    %\KwIn{$K$,$T$,$M$}
    $D=\lceil T/MK\rceil$, $l_0=\lfloor \log_4\left(\frac{3D}{67K\log(MKT)}\right)\rfloor$, $m_l=\lfloor 4^{l+3}\log(MKT)\rfloor$\\
    \tcc{Stage 1: Separate Burn-in}
    \For{Agent $i=1,\cdots,M$}{
        Agent $i$ runs single-agent elimination for $D$ time steps, denote remaining arms as $A^{(i)}$%$A^{(i)}=$\texttt{Eliminate}$([K],D)$       \tcp*[r]{Run single agent elimination for $D$ time steps}
    }
    \BlankLine
    \BlankLine
    \tcc{Switching: Random Allocation}
    Generate public random numbers $r_1$,...,$r_K$ uniformly distributed in $[M]$\\
    $B_{l_0+1}^{(i)}=\left\{a\in A^{(i)}|r_a=i\right\}$, $B_{l_0+1}=\bigcup_{i\in[M]}B_{l_0+1}^{(i)}$\\
    \BlankLine
    \BlankLine
    \tcc{Stage 2: Distributed Elimination}
    \For{$l=l_0+1,\cdots$}{
        \eIf{$N_l=\left|B_l\right|>M$}{
            Agent $i$ sends $n_l^{(i)}=\left|B_{l}^{(i)}\right|$ to server; server broadcasts $n_{max}=\max_i n_l^{(i)}$\\
            \If{$(n_{l}^{(1)},...,n_{l}^{(M)})$ is not balanced}{\texttt{Reallocate}    \tcp*[r]{Adjust $B_l^{(i)}$ such that their sizes are balanced}}
            Agent $i$ pulls each $a\in B_l^{(i)}$ for $m_l$ times, denotes average reward as $\hat{u}_l(a)$, and then pulls arms in round-robin for $(n_{max}-n_l^{(i)})m_l$ times before the next communication round \\
            %Agent $i$ pulls arms in round-robin for $(n_{max}-n_l^{(i)})m_l$ times before the next communication round\\
            Communication round: Agent $i$ sends $\max_{a\in B_{l}^{(i)}}\hat{u}_l$ to server and waits to receive $u^*_l=\max_{a\in B_{l}}\hat{u}_l$ from server\\
            Agent $i$ eliminates bad arms: $B_{l+1}^{(i)}=\left\{a\in B_l^{(i)}:\hat{u}_l(a)+2^{-l}\ge u^*_l\right\}$
        }{
            For each arm in $B_l$, the server asks $M/N_l$ agents to pull it $m_lN_l/M$ times\\
            Server computes $\hat{u}_l(a)$, the average reward for $m_l$ pulls of arm $a$\\
            Server eliminates bad arms: $B_{l+1}=\left\{a\in B_l:\hat{u}_l(a)+2^{-l}\ge \max_{b\in B_l}\hat{u}_l(b)\right\}$
        }
    }
\end{algorithm}

The DEMAB protocol executes in two stages. In the first stage, each agent directly runs the single-agent elimination algorithm for $D=\lceil T/MK\rceil$ time steps. The remaining arms of agent $i$ are denoted as $A^{(i)}$. In $D$ time steps, an agent completes at least $l_0=\lfloor \log_4\frac{3D}{67K\log T}\rfloor$ phases. The purpose of this separate burn-in period is to eliminate the worst arms quickly without communication, so that in the second stage, elimination can begin with a small threshold of $O(2^{-l_0})$. $D$ and $l_0$ is chosen so that the total regret within the first stage is $\Tilde{O}\left(\sqrt{MT}\right)$.

Between the two stages, the remaining arms are randomly allocated to agents. Public randomness is used to allocate the remaining arms to save communication. Agents first generate $(r_1,\cdots,r_K)$, $K$ uniformly random numbers in $[M]$, from a public random number generator. Agent $i$ then computes $B^{(i)}=\left\{a\in A^{(i)}|r_a=i\right\}$. By doing so, agent $i$ keeps each arm in $A^{(i)}$ with probability $1/M$, and the resulting sets $B^{(1)},\cdots,B^{(M)}$ are disjoint. Meanwhile, every arm in $\bigcap_{i\in[M]}A^{(i)}$ is kept in $B_{l_0+1}=\bigcup_{i\in[M]}B^{(i)}$, so that the best arm remains in $B_{l_0+1}$ with high probability\footnote{$\bigcup_{i\in[M]}A^{(i)}$ may not be a subset of $B_{l_0+1}$, which is not a problem in the regret analysis.}.%agent $i$ samples a random subset $B^{(i)}$ from $A^{(i)}$. Each arm in $A^{(i)}$ is kept with $1/M$ probability. Shared random numbers is used to guarantee $B^{(1)}$,...,$B^{(M)}$ are disjoint.

In the second stage, agents start to simulate a single-agent elimination algorithm starting from phase $l_0+1$. Initially, the arm set is $B_{l_0+1}$. In phase $l$, each arm in $B_l$ will be pulled for at least $m_l=\lceil 4^{l+3}\log(MKT)\rceil$ times. Denote the average reward of arm $a$ in phase $l$ by $\hat{u}_l(a)$. If $\hat{u}_l(a)<\max_{a'\in B_l}\hat{u}_l(a')-2^{-l}$, it will be eliminated; the arm set after the elimination is $B_{l+1}$.

This elimination in the second stage is performed over $M$ agents in two ways: In \emph{distributed mode} or in \emph{centralized mode}. Let $N_l=\left|B_l\right|$ be the number of remaining arms at the start of phase $l$. If $N_l$ is larger than $M$, the elimination is performed in \emph{distributed mode}. That is, agent $i$ keeps a set of arms $B_l^{(i)}$, and pulls each arm in $B_l^{(i)}$ for $m_l$ times in phase $l$. Each agent only needs to send the highest average reward to the server, who then computes and broadcasts $\max_{a\in B_l}\hat{u}_l(a)$. Agent $i$ then eliminates low-rewarding arms from $B_l^{(i)}$ on its local copy.

When $N_l\le M$, the elimination is performed in \emph{centralized mode}. That is, $B_l$ will be kept and updated by the server\footnote{In the conversion from distributed mode to centralized mode, agents send their local copy to the server, which has $O(M)$ communication cost.}. In phase $l$, the server assigns an arm in $B_l$ to ${M}/{N_l}$ agents, and asks each of them to pull it ${m_lN_l}/{M}$ times\footnote{The indivisible case is handled in Appendix ~\ref{sec:detail_DEMAB} .}. The server waits for the average rewards to be reported, and then performs elimination on $B_l$.

One critical issue here is \emph{load balancing}, especially in distributed mode. Suppose that $n_l^{(i)}=\left|B_l^{(i)}\right|$, $n_{max}=\max_{i\in[M]}n_l^{(i)}$. Then the length of phase $l$ is determined by $n_{max}m_l$. Agent $i$ would need to keep pulling arms for $\left(n_{max}-n_l^{(i)}\right)m_l$ times until the start of the next communication round. %end of this phase. %next synchronization step. %in order to wait for this slowest agent. %In distributed mode, the length phase $l$ is determined by the agent with the largest $B_l^{(i)}$. In the synchronized setting we consider, other agents must pull arms while waiting for this agent. 
This will cause an arm to be pulled for much more than $m_l$ times in phase $l$, and can hurt the performance. Therefore, it is vital that at the start of phase $l$, $\Vec{n}_l:=(n_l^{(1)},...,n_l^{(M)})$ is balanced\footnote{By saying a vector of numbers to be balanced, we mean the maximum is at most twice the minimum.}.

The subroutine \texttt{Reallocate} is designed to ensure this by reallocating arms when $\Vec{n}_l$ is not balanced. First, the server announces the total number of arms; then, agents with more-than-average number of arms donate surplus arms to the server; the server then distributes the donated arms to the other agents, so that every agent has the same number of arms. However, calling \texttt{Reallocate} is communication-expensive: it takes $O\left(\min\{N_l, N_{l'}-N_l\}\right)$ communication cost, where $l$ is the current phase and $l'$ is the last phase where \texttt{Reallocate} is called. Fortunately, since $\left\{B_{l_0+1}^{(i)}\right\}_{i\in[M]}$ are generated randomly, it is unlikely that one of them contain too many good arms or too many bad arms. By exploiting shared randomness, we greatly reduce the expected communication cost needed for load balancing.
%However, calling \texttt{Reallocate} is communication-expensive: it could cost as much as $K$. Fortunately, since $\left\{B_{l_0+1}^{(i)}\right\}_{i\in[M]}$ are generated randomly, it is unlikely that one of them contain too many nice arms or too many bad arms. By exploiting shared randomness, we can greatly reduce the communication cost needed for load balancing.

%Detailed descriptions of \texttt{Eliminate} and \texttt{Reallocate}, as well as a more detailed description of DEMAB are provided in the supplementary material.

Detailed descriptions of the single-agent elimination algorithm, the \texttt{Reallocate} subroutine, and the DEMAB protocol are provided in Appendix~\ref{sec:detail_DEMAB}.

Access to a public random number generator, which is capable of generating and sharing random numbers with all agents with negligible communication cost, is assumed in DEMAB. This is not a strong assumption, since it is well known that a public random number generator can be replaced by private random numbers with a little additional communication~\cite{newman1991private}. In our case, only $O\left(M\log T\right)$ additional bits of communication, or $O(M)$ additional communication cost, are required for all of our theoretical guarantees to hold. See Appendix~\ref{sec:proof_1} for detailed discussion.
%Access to a public random number generator, which is capable of generating and sharing random numbers with all agents with negligible communication cost, is assumed in DEMAB. This is not a strong assumption. As shown by Newman's theorem~\cite{newman1991private}, public randomness can be replaced by private randomness and a little additional communication. In our case, only $O\left(M\log T\right)$ additional bits of communication are required (see supplementary materials detailed discussion) for all of our theoretical guarantees to hold. %since in practice, this can be achieved by using a pseudo-random number generator and broadcasting a short random seed. This assumption can also be removed. As shown by Newman's theorem~\cite{newman1991private}, public randomness can be realized by private randomness and a little additional communication. In our case, only $O\left(M\log T\right)$ additional bits of communication are required (see supplementary materials detailed discussion).

\subsection{Regret and Communication Efficiency of DEMAB}
\label{sec:mabregret}
In this subsection, we show that the DEMAB protocol achieves near-optimal regret with efficient communication, as captured by the following theorem.
\begin{theorem}
\label{thm:mabregret}
The DEMAB protocol incurs $O\left(\sqrt{MTK\log T}\right)$ regret, $O\left(M\log\frac{MK}{\delta}\right)$ communication cost with probability $1-\delta$, and $O\left(M\log(MK)\right)$ communication cost in expectation.
\end{theorem}
The worst-case regret bound above can be improved to an instance-dependent near-optimal regret bound by changing the choice of $D$ and $l_0$ to $0$. In that case the communication cost is $O(M\log T)$, which is a small increase. See Theorem \ref{thm:DEMAB_instance_dependent} in Appendix~\ref{sec:proof_1} for detailed discussion.

We now give a sketch of the proof of Theorem~\ref{thm:mabregret}.
\paragraph{Regret} In the first stage, each agent runs a separate elimination algorithm for $D$ timesteps, which has regret $O(\sqrt{KD\log D})$. Total regret for all agents in this stage is $O(M\sqrt{KD\log D})=O(\sqrt{MT\log T})$. After the first stage, each agent must have completed at least $l_0$ phases. Hence, with high probability, before the second stage, $B_{l_0+1}=\bigcup_{i\in [M]}B_{l_0+1}^{(i)}$ contains the optimal arm and only arms with suboptimality gap less than $2^{-l_0+1}$.

In the second stage, if $a\in B_l$, it will be pulled at most $2m_l$ times in phase $l$ because of our load balancing effort. Therefore, if arm $k$ has suboptimality gap $0<\Delta_k<2^{-l_0+1}$, it will be pulled for $\Tilde{O}\left(\Delta_k^{-2}\right)$ times. It follows that regret in the second stage is $O\left(\sqrt{MKT\log T}\right)$, and that total regret is $O\left(\sqrt{MKT\log T}\right)$.

\paragraph{Communication} In the first stage and the random allocation of arms, no communication is needed. The focus is therefore on the second stage.

During a phase, apart from the potential cost of calling \texttt{Reallocate}, communication cost is $O\left(M\right)$. The communication cost of calling \texttt{Reallocate} in phase $l$ is at most $O\left(\min\left\{N_l,N_{l'}-N_{l}\right\}\right)$, where $l'$ is the last phase where \texttt{Reallocate} is called. Therefore, total cost for calling \texttt{Reallocate} in one execution is at most $O\left(N_{l_1}\right)$, where $l_1$ is the first phase in which \texttt{Reallocate} is called. From the definition of $m_l$ and $l_0$, we can see that there are at most $L=O\left(\log(MK)\right)$ phases in the second stage. Therefore in the worst case, communication cost is $O\left(ML+K\right)$ since $N_{l_1} \leq K$.

However, in expectation, $N_{l_1}$ is much smaller than $K$. Because of the random allocation, when $N_l$ is large enough, $\Vec{n}_l$ would be balanced with high probability. In fact, with probability $1-\delta$, $N_{l_1}=O\left(M\log \frac{MK}{\delta}\right)$. Setting $\delta=1/K$, we can show that the expected communication complexity is $O(M \log MK)$.

\subsection{Lower Bound}
\label{sec:mablower}
Intuitively, in order to avoid a $\Omega\left(M\sqrt{KT}\right)$ scaling of regret, $\Theta(M)$ amount of communication cost is necessary; otherwise, most of the agents can hardly do better than a single-agent algorithm. We prove this intuition in the following theorem.%following lower bound (which applies to the expected communication cost of randomized protocols).
\begin{theorem}
\label{thm:lowerbound}
For any protocol with expected communication cost less than $M/3000$, there exists a MAB instance such that total regret is $\Omega(M \sqrt{KT})$. 
\end{theorem}
The theorem is proved using a reduction from single-agent bandits to multi-agent bandits, i.e. a mapping from protocols to single-agent algorithms.
%The proof of this lower bound is a simple reduction from single-agent bandits to multi-agent bandits, mapping protocols to single-agent algorithms. %Since one can simulate MAB with linear bandits by setting an action set of $K$ orthogonal vectors, Theorem~ induces a $\Omega(M\sqrt{dT})$ lower bound for linear bandits when communication complexity is less than $M/c$.

One can trivially achieve $O(M\sqrt{KT})$ regret with $0$ communication cost by running an optimal MAB algorithm separately. Therefore, Theorem~\ref{thm:lowerbound} essentially gives a $\Omega(M)$ lower bound on communication cost for achieving non-trivial regret. The communication cost of DEMAB is only slightly larger than this lower bound, but DEMAB achieves near-optimal regret. This suggests that the communication-regret trade-off for distributed MAB is a steep one: with $O\left(M\log MK\right)$ communication cost, regret can be near-optimal; with slightly less communication, regret necessarily deteriorates to the trivial case.
%The communication cost Theorem~\ref{thm:mabregret} matches the lower bound of 
%We note that Theorem~\ref{thm:mabregret} matches the lower bound of $\Omega\left(M\right)$ except for a factor of $\log MK$. In this sense our protocol for multi-armed bandits is indeed communication-efficient. Moreover, since $O(M\sqrt{KT})$ regret can be achieved with no communication at all, it is essentially a trivial regret scaling for a communication protocol. Therefore, this lower bound suggests that the communication-regret trade-off for distributed MAB is a steep one: with $O\left(M\log MK\right)$ communication cost, regret can be near-optimal; with slightly less communication ($M/c$), regret necessarily deteriorates to the trivial case.
\section{Main Results for Linear Bandits}
\label{sec:liear}
%In this section, we propose two protocols for linear bandits.
%\subsection{Protocol for Linear bandits with Fixed Action Set}
\label{sec:linearelim}
In this section, we summarize the single-agent elimination algorithm (algorithm 12 in \cite{lattimore2018bandit}), and present the Distributed Elimination for Linear Bandits (DELB) protocol. This protocol is designed for the case where the action set $\mathcal{D}$ is fixed, and has communication cost with almost linear dependence on $M$ and $d$. Our results for linear bandits with time-varying action set is presented in Sec.~\ref{sec:linearucb}. For convenience, we assume $ \mathcal{D} $ is a finite set, which is without loss of generality\footnote{When $\mathcal{D}$ is infinite, we can replace $\mathcal{D}$ with an $\epsilon$-net of $\mathcal{D}$, and only take actions in the $\epsilon$-net. If $\epsilon<1/T$, this will not influence the regret. This is a feasible approach, but may not be efficient.}. % with time-invariant action set $D$, and enjoys almost linear dependence on $M$ and $d$.
\subsection{Elimination Algorithm for Single-agent Linear Bandit}
The elimination algorithm for linear bandits~\cite{lattimore2018bandit} also iteratively eliminates arms from the initial action set. In phase $l$, the algorithm maintains an active action set $A_l$. It computes a distribution $\pi_l(\cdot)$ over $A_l$ and pulls arms according to $\pi_l(\cdot)$. Suppose $n_l$ pulls are made in this phase according to $\pi_l(\cdot)$. We use linear regression to estimate the mean reward of each arm based on these pulls. Arms with estimated rewards $2^{-l+1}$ lower than the maximum are eliminated at the end of the phase.

 To eliminate arms with suboptimality gap $2^{-l+2}$ in phase $l$ with high probability, the estimation error in phase $l$ needs to be less than $2^{-l}$. On the other hand, to achieve tight regret, the number of pulls we make in phase $l$ needs to be as small as possible. 
 %The main technical problem here is how to find a policy $\pi_l(\cdot)$ with finite support size so that the number of pulls to construct such a confidence set is small. 
 Let $V_l(\pi)=\sum_{x\in \mathcal{A}_l}\pi(x)xx^{\top}$ and $g_l(\pi)=\max_{x\in \mathcal{A}_l}x^{\top}V_l(\pi)^{-1}x$. According to the analysis in \cite{lattimore2018bandit} (Chapter 21), if we choose each arm $x \in \operatorname{Supp}(\pi_l)$ exactly $\left\lceil \pi(x) g_l(\pi) 4^{l} \log \left(\frac{1}{\delta}\right)\right\rceil$ times, the estimation error for any arm $x\in A_l$ is at most $2^{-l}$ with high probability. This means that we need to find a distribution $\pi_l(\cdot)$ that minimizes $g_l(\pi)$, which is equivalent to a well-known problem called $G$-optimal design~\cite{pukelsheim2006optimal}. One can find a distribution $\pi^{*}$ minimizing $ g $ with $g(\pi^{*})=d$. The support set of $\pi^*$ (a.k.a. the core set) has size at most $d(d+1)/2$. As a result, only $\sum_{x \in \operatorname{Supp}(\pi^*_l)}\lceil \pi^*(x)g_l(\pi^*)4^l \log(\frac{1}{\delta})\rceil \leq O(4^{l}d\log(\frac{1}{\delta})+d^2)$ pulls are needed in phase $l$. 

%The main technical problem is how to find a policy $\pi_l(\cdot)$ with finite support size to make the confidence set small. Suppose $V(\pi)=\sum_{a\in \mathcal{D}}\pi(a)aa^T,$ $g(\pi)=\max_{a\in \mathcal{D}}a^TV(\pi)^{-1}a$. Since the suboptimality gap of the confidence set depends on $V_{\pi}$, our aim is to find a policy $\pi$ with finite support to minimize $V(\pi)$. According to the analysis in \cite{lattimore2018bandit}, this is equivalent to $G$-optimal design problem~\cite{pukelsheim2006optimal}. In this problem, we require $g(\pi)$ to be actually minimized, and the minimum value is $d$~\cite{kiefer1960equivalence}. It is known that there is a solution $\pi^*$ such that the support of $\pi^*$ (also known as core set in literature) has size at most $\xi=d(d+1)/2$~\cite{todd2016minimum}. This helps us construct a confidence set with suboptimality gap $2^{-l}$, and only $\xi=d(d+1)/2$ different arms are pulled in phase $l$.
 \subsection{The DELB Protocol}
 \label{sec:delb}
 %We introduce our protocol at first, then analyze how we implement elimination algorithm in distributed mode. 
In this protocol, we parallelize the data collection part of each phase by sending instructions to agents in a communication-efficient way. %and send instructions in a communication-efficient way.
In phase $l$, the server and the agents both locally solve the same $G$-optimal design problem on $A_l$, the remaining set of actions. We only need to find a $2$-approximation to the optimal $g(\pi)$. That is, we only need to find $\pi_l(\cdot)$ satisfying $g(\pi) \leq 2d$. On the other hand, we require the solution to have a support smaller than $\xi = 48 d \log \log d$. This is feasible since the Frank-Wolfe algorithm under appropriate initialization can find such an approximate solution for finite action sets $\mathcal{D}$ (see Proposition 3.17~\cite{todd2016minimum}).  After that server assigns arms to agents. Since both the server and agents obtain the same core set by solving $G$-optimal design, the server only needs to send the index among $\xi$ arms to identify and allocate each arm. After pulling arms, agents send the results to the server, who summarizes the results with linear regression. Agents and the server then eliminate low rewarding arms from their local copy of $A_l$.%and we finish this phase.

For convenience, we define $V(\pi)=\sum_{x\in \mathcal{D}}\pi(x)xx^{\top},$ $g(\pi)=\max_{x\in \mathcal{D}}x^{\top}V(\pi)^{-1}x.$%\footnote{When $\mathcal{D}$ is infinite, the sum is calculated on $\operatorname{Supp}(\pi)$, which is finite . }
%\begin{wraptable}{r}{7cm}

\begin{algorithm}[H]
	\DontPrintSemicolon
	\label{alg:elms_linear}
	%\begin{algorithmic}[5]
	$A_1=\mathcal{D}$, $C_1=600$\\
	\For{$l=1,2,3,...$}{
	    \tcc{All agents and server: Solve a G-optimal design problem}
	     Find distribution $\pi_l(\cdot)$ over $A_l$ such that:
		1. its support has size at most $\xi=48d\log\log d$; 2. $g(\pi)\leq 2d$\\
		\tcc{Server: Assign pulls and summarize results}
		Assign $m_l(x)=\lceil C_14^{l}d^2\pi_l(x)\ln MT \rceil$ pulls for each arm $x\in \operatorname{Supp}(\pi_l)$ and wait for results\footnotemark. %$m_l(x)=\lceil C_14^{l}d^2\pi_l(x)\ln MT \rceil$\\
		
		Receive rewards for each arm $x\in A_l$ reported by agents\\
		For each arm in the support of $\pi_l(\cdot)$, calculate the average reward $\mu(x)$\\

		Compute\footnotemark
		$X=\sum_{x\in \operatorname{Supp}(\pi_l)} m_l(x)\mu(x)x$, $V_l=\sum_{x\in \operatorname{Supp}(\pi_l)} m_l(x)xx^{\top}$, $\hat{\theta}=V_l^{-1}X$ \\
		Send $\hat{\theta}$ to all agents\\
		\tcc{All agents and server: Eliminate low-rewarding arms}
		Eliminate low rewarding arms:
		$A_{l+1}=\left\{x\in A_l:\max_{b\in A_l}\langle\hat{\theta},b-x\rangle\leq 2^{-l+1}\right\}$
	}
	%\end{algorithmic}
	\caption{Distributed Elimination for Linear Bandits (DELB)}
\end{algorithm}

%\end{wraptable}

\addtocounter{footnote}{-2} %3=n
 \stepcounter{footnote}
\footnotetext{We assign the pulls of each arm to as few agents as possible. See Appendix~\ref{sec:detail_DELB} for detailed description.}
 \stepcounter{footnote}\footnotetext{$V_l$ is always invertible when $A_l$ spans $R^d$. When $A_l$ doesn't span $R^d$, we can always consider $\operatorname{span}(A_l)$ in phase $l$ and reduce the number of dimensions.}

\subsection{Regret and Communication Efficiency of DELB}
We state our results for the elimination-based protocol for distributed linear bandits. The full proof is given in Appendix \ref{sec:proof3}.
\begin{theorem}
\label{thm:linearelim}
The DELB protocol achieves expected regret $O\left(d\sqrt{TM\log T}\right)$ with communication cost $O\left((Md+d\log\log d)\log T\right)$.
\end{theorem}

\emph{Proof sketch: }
 In round $l$, the number of pulls is at most $48d \log \log d + C_1 4^l d^2 \log MT$. Based on the analysis for elimination-based algorithm, we can show that the suboptimality gap $\left\langle\theta^{*}, x^*-x\right\rangle$ is at most $2^{-l+2}$ with probability $1-1/MT$ for any arm $x$ pulled in phase $l$. Suppose there are at most $L$ phases, we can prove that
$
\mathbb{E}(REG(T)) \leq \sum_{l=1}^{L}O(d\sqrt{4^ld^2\log ^2TM})\leq O(d\sqrt{TM\log TM}).
$
%Apparently $C_14^Ld^2\log TM\leq TM$. Therefore

In each phase, communication cost comes from three parts: assigning arms to agents, receiving average rewards of each arm and sending $\hat{\theta}$ to agents. In the first and second part, each arm $x \in \operatorname{Supp}(\pi_l)$ is designated to as few agents as possible. We can show that the communication cost of these parts is $O(M+d\log \log d)$. In the third part, the cost of sending $\hat{\theta}$ is $Md$. Since $l$ is at most $O(\log T)$, the total communication is
$O\left((Md+d\log\log d)\log T\right).$
%Under the assumption that $T>M$, this can be simplified to $O\left((Md+d\log\log d)\log T\right)$.
\qed

%Note that the replicated calculation of $\pi_l(\cdot)$ on the agent's side is a trick purposed to reduce the number of bits/packets required to designate an arm. If this trick is not used, the server may need $d$ numbers, instead of one, to identify an arm; therefore the total communication cost will become $O\left((Md+d^2\log\log d)\log T\right)$.

\subsection{Protocol for Linear Bandits with Time-varying Action Set}
\label{sec:linearucb}

In some previous work on linear bandits~\cite{chu2011contextual, abbasi2011improved}, the action set available at timestep $t$ may be time-varying. That is, players can only choose actions from $\mathcal{D}_t$ at time $t$, while regret is defined against the optimal action in $\mathcal{D}_t$. The DELB protocol does not apply in this scenario. To handle this setting, we propose a different protocol DisLinUCB (Distributed LinUCB) based on LinUCB~\cite{abbasi2011improved}. We only state the main results here. Detailed description of the protocol and the proof are given in Appendix \ref{sec:detail_dislinucb} and Appendix \ref{sec:proof4}.

\begin{theorem}
DisLinUCB protocol achieves expected regret of $O\left(d\sqrt{MT}\log^{2}(T)\right)$ with $O\left(M^{1.5}d^3\right)$ communication cost.
\end{theorem}
Although the regret bound is still near-optimal, the communication cost has worse dependencies on $M$ and $d$ compared to that of DELB. %It is conjectured that this is a price paid for its flexibility.

\bibliographystyle{plain}    %% We can also choose IEEEtran
\bibliography{ref.bib}

\clearpage
\setcounter{section}{0}
\appendix
\renewcommand{\appendixname}{Appendix~\Alph{section}}

\section{Detailed Description of DEMAB}
\label{sec:detail_DEMAB}
In this section, we give a detailed description of the DEMAB protocol and some subroutines used in the protocol.

\begin{algorithm}[H]
	\caption{Distributed Elimination for Multi-armed Bandits (DEMAB)}
	\DontPrintSemicolon
	$D=\lceil T/MK\rceil$, $C_2=67/3$, $l_0=\lfloor \log_4\left(\frac{D}{C_2K\log (MKT)}\right)\rfloor$, $m_l=\lceil 4^{l+3}\log (MKT)\rceil$\\
	\tcc{Stage 1: Separate Burn-in}
    \For{agent $i=1,...,M$}{
		$A^{(i)}=$\texttt{Eliminate}($[K]$,$D$)
	}
    \tcc{Switching: Random Allocation}
	Generate public random numbers $r_1,...,r_K$ in $[M]$\\
    $B^{(i)}_{l_0+1}=\{a\in A^{(i)}|r_a=i\}$\\
	%Agents send $n_{l_0}^{(i)}=\left|B_{l_0+1}^{(i)}\right|$ to server, $N_{l_0+1}=\sum_i n_{l_0}^{(i)}$, $N_{max}=\max_i n_{l_0}^{(i)}$\\
	\tcc{Stage 2: Distributed Elimination}
	
	\For{$l=l_0+1,...$}{
	    %Server sends $N_l,N_{max}$ to all agents\\
	    \eIf{\texttt{Centralize} has not been called}{
	        \tcc{distributed mode}
	        Agents send $n_{l}^{(i)}=\left|B_{l}^{(i)}\right|$ to server, $N_{l}=\sum_i n_{l}^{(i)}$, $N_{max}=\max_i n_{l}^{(i)}$\\
	        \If{$N_l\le M$}{\texttt{Centralize}, go to line 22}
	        \If{$\Vec{n}_{l}$ is not balanced}{\texttt{Reallocate}}
	        Server sends $N_{max}$ to all agents\\
		    \For{agent $i=1,...,M$}{
		        Pull each $a\in B_l^{(i)}$ for $m_l$ times, denote average as $\hat{u}_l(\cdot)$\\
		        Pull other arms in round-robin for $(N_{max}-\left|B_l\right|)m_l$ times\\
		        Send $(\arg\max_{a'} \hat{u}_l(a'), \max_{a'} \hat{u}_l(a'))$ to server
		    }
            Server receives $(a^*_{j,l},u^*_{j,l})$ from agent $j$, and sends $u^*_l=\max_{j}u^*_{j,l}$ to every agent\\
		        \For{agent $i=1,...,M$}{
		            Elimination: $B^{(i)}_{l+1}=\left\{i\in B^{(i)}_l:\hat{u}_l(a)+2^{-l}\geq u^*_l\right\}$\\
    		        %Send $n_l(i)=\left|B_{l+1}^{(i)}\right|$ to server\\
		        }
		        %\If{$\Vec{n}_{l}$ is not balanced}{\texttt{Reallocate}}
	    }{
	    \tcc{centralized mode}
		        Server assigns arms in $B_l$ to agents evenly and schedules $m_l$ pulls for each arm\\
		        Agents pull arms as required by the server, and report the average for the pulled arm\\
		        Server calculates $\hat{u}_l(a)$, average reward for $m_l$ pulls in this phase, for each arm $a\in B_l$\\
		        Elimination: $B_{l+1}=\left\{a\in B_l: \hat{u}_l(a)+2^{-l}\geq \max_{j\in B_l}\hat{u}_l(j)\right\}$\\
	    }
	}
\end{algorithm}

\textbf{Eliminate:} \texttt{Eliminate} executes the single-agent elimination algorithm. In this function, each agent runs the single-agent elimination algorithm for $ D $ time steps, then return the remaining arms.
\begin{algorithm}[H]
    \label{alg:elim}
    \caption{\texttt{Eliminate}}
    \KwIn{A set of arms $A_1$, time step $ D $.}
    \For{$l=1,...$}{
        \For{$a\in A_l$}{
            Pull arm $a$ for $m_l=\lceil 4^{l+3}\log(MKT)\rceil$ times, denote average reward as $\hat{\mu}_{l}(a)$\\
            If time step $D$ is reached, go to line 6\\
            Elimination: $A_{l+1}=\left\{a\in A_l: \hat{\mu}_l(a)>\max_{k\in A_l}\hat{\mu}_l(k)-2^{-l}\right\}$\\
        }
    }
    Return $A_l$
    %\begin{algorithmic}[1]
    %\State Input: A set of arms $A_1$, time step $ D $.
    %		\For{$l=1,...$}
	%		\For{$a\in A_l$}
	%			\State Pull arm $a$ for $m_l$ times, denote average reward as $\hat{\mu}_l(a)$;
	%			\State If timestep $D$ is reached, go to line 9;
	%			\State $A_{l+1}=\left\{a\in A_l: \hat{\mu}_l(a)>\max_{k\in A_l}\hat{\mu}_l(k)-2^{-l}\right\}$
	%		\EndFor
	%	\EndFor
	%	\State Output: $A_l$
    %\end{algorithmic}
\end{algorithm}

\textbf{Reallocate:} In \texttt{Reallocate}, the server announces the average number of arms; agents with more-than-average arms then donate surplus arms to the server; the server then distributes the donated arms to the other agents, so that every agent has nearly the same number of arms. After calling \texttt{Reallocate}, $\Vec{n}_{l}$ becomes balanced again. The function contains the following two parts: One running on the server, and the other running on each agent.

\begin{algorithm}[H]
    \label{alg:server_real}
    \caption{\texttt{Reallocate} for Server}
    
    \KwIn{ $n_l^{(1)}$,...,$n_l^{(M)}$}
    $\bar{n}=\lfloor \sum_{i=1}^M n_l^{(i)} /M \rfloor$\\
    Send ``reallocation'', $\bar{n}$ to every player\\
    Receive a set of $n_l^{(i)}-\bar{n}$ arms, $A'_i$, from player $i$ if $n_l^{(i)}>\bar{n}$; $A_{temp}=\bigcup_{i}A'_i$\\
    \For{$i=1,...,M$}{
        If $n_l^{(i)}<\bar{n}$, send $n_l^{(i)}-\bar{n}$ arms in $A_{temp}$ to player $i$, and remove them from $A_{temp}$
    }
    Send the remaining arms in $A_{temp}$ to players $1$,...,$|A_{temp}|$ (one each)
    %\begin{algorithmic}[1]
    %\State Input: $n_l(1)$,...,$n_l(M)$
	%\State $\bar{n}=\lfloor \sum_{i=1}^M n_l(i) /M \rfloor$
	%\State Send ``reallocation'', $\bar{n}$ to every player
	%\State Receive a set of $n_l(i)-\bar{n}$ arms, $A'_i$, from player $i$ if $n_l(i)>\bar{n}$; $A_{temp}=\bigcup_{i}A'_i$
	%\For{$i=1,...,M$}
	%	\State If $n_l(i)<\bar{n}$, send $n_l(i)-\bar{n}$ arms in $A_{temp}$ to player $i$, and remove them from $A_{temp}$
	%\EndFor
	%\State Send the remaining arms in $A_{temp}$ to players $1$,...,$|A_{temp}|$ (one each)
    %\end{algorithmic}
\end{algorithm}

\begin{algorithm}[H]
	\caption{\texttt{Reallocate} for Agents}
    \If{Receive ``reallocation'', $\bar{n}$}{
        \eIf{$n_l^{(i)}>\bar{n}$}{
            Pick a subset of $n_l^{(i)}-\bar{n}$ arms, $A'_i$, from $B_{l}^{(i)}$, and send them to server\\
			$B_{l}^{(i)}=B_{l}^{(i)}\setminus A'_i$
        }{
        Wait until receiving $A'_i$ from server, $B_{l}^{(i)}=B_{l}^{(i)}\cup A'_i$
        }
    }
\end{algorithm}

%\begin{algorithm}[H]
%	\caption{\texttt{Reallocate} for Agents}
%	\begin{algorithmic}[1]
%		\If{Receive ``reallocation'', $\bar{n}$}
%			\If{$n_l(i)>\bar{n}$}
%				\State Pick a subset of $n_l(i)-\bar{n}$ arms, $A'_i$, from $B_{l+1}^{(i)}$, and send to server
%				\State $B_{l+1}^{(i)}=B_{l+1}^{(i)}\setminus A'_i$
%			\Else
%				\State Wait until receiving $A'_i$ from server, $B_{l+1}^{(i)}=B_{l+1}^{(i)}\cup A'_i$
%			\EndIf
%		\EndIf
%	\end{algorithmic}
%\end{algorithm}

\textbf{Centralize:} When the number of arms drops below $M$, the subroutine \texttt{Centralize} is called, in which agents send their local copy of remaining arms, $B_l^{(i)}$, to the server, and server receives $B_l=\bigcup_{i\in [M]} B_l^{(i)}$.

\begin{algorithm}[H]
    \caption{\texttt{Centralize}}
    \For{agent $i=1,...,M$}{Send $B_l^{(i)}$ to server}
    Server receives $B_l^{(i)}$ from agent $i$, and calculates $B_{l}=\bigcup_{i}B_l^{(i)}$\\
\end{algorithm}
%\begin{algorithm}[H]
%	\caption{\texttt{Centralize}}
%	\begin{algorithmic}[1]
%		\For{agent $i=1,...,M$}
%		    \State Send $B_l^{(i)}$ to server
%		\EndFor
%		\State Server receives $B_l^{(i)}$ from agent $i$, and calculates $B_{l}=\bigcup_{i}B_l^{(i)}$
%	\end{algorithmic}
%\end{algorithm}

\textbf{Assignment Strategy:} In centralized mode, server assigns arms to agents in the following way. Let $N_l=|B_l|$. If $M$ is exactly divisible by $N_l$, for each arm in $B_l$, server asks $M/N_l$ separate agents to play it for $\lceil m_l N_l/M \rceil$ times. If not, we allocate pulls to agents in the following way: Let $p_l=\lceil m_lN_l/M\rceil$ denote the average pulls each agent needs to perform. Our assignment starts from the arm with the smallest index $a_{k_1}$ and agent 1. For arm $a_{k_j}$ and agent $i$, if agent $i$ has been assigned $p_l$ pulls, we turn to agent $i+1$. If we have finished allocating $m_l$ pulls for arm $a_{k_j}$, we continue designating arm $a_{k_{j+1}}$. The assignment is finished until all pulls are scheduled to agents. %Note that the communication cost of this strategy is $O(M+N_l)$ in phase $l$.

%\textbf{Assignment Strategy:} In the second part of stage 3, we assign pulls to agents in the following way: Let $p_l=\lceil m_l N_l/M\rceil$ denote the average pulls each agent needs to perform. Our assignment starts from the arm with the smallest index number and agent 1. For arm $a_k$ and agent $i$, if agent $i$ has been assigned $p_l$ pulls, we turn to agent $i+1$. If we have finished designating $m_l$ pulls for arm $a_k$, we continue designating arm $a_{k+1}$. The assignment is finished until all pulls are scheduled to agents. The communication cost of this strategy is $O(M+N_l)$ in phase $l$.

\section{Proof of Theorem 1}
\label{sec:proof_1}
In this section, we give a full proof of Theorem 1, which bounds the total regret and communication cost of the DEMAB protocol. In the analysis below, we will use $\Tilde{B}_{l}$ to represent $B_{l}^{(1)}\cup \cdots \cup B_{l}^{(M)}$ (in the distributed mode) or $B_l$ (in the centralized mode). It refers to the set of remaining arms at the start of the $l$-th phase at stage 2, either held separately by the agents or held by the server. %Note that we introduce $\tilde{B}_l$ to simplify our analysis: in distributed mode, $ B_l = \emptyset $ and $ B_l^{(i)} $ is held locally by each agent; in centralized mode, $ B_l $ is maintained by the server, and $ B_l^{(i)} = \emptyset $ for each agent.

Suppose that the protocol terminates when $l=l'$.
%for convenience, we define $\Tilde{B}_{l}$ ($l'>l$) to be the same as $\Tilde{B}_{l'}$.
We also let $N_M(a,t)=\sum_{j=1}^{t}\sum_{i=1}^M \mathbb{I}\left[a_{i,j}=a\right]$ be the number of times arm $a$ is pulled before time step $t$. Without loss of generality, we assume that arm $1$ is the best arm, and define $\Delta_k := \mu(1)-\mu(k)$.

We first state a few facts and lemmas.
\begin{fact}
We state some facts regarding the execution of the algorithm.
\begin{enumerate}
    \item At line $6$ of the server's part in \texttt{Reallocate}, $|A_{temp}|<M$;
    \item After \texttt{Reallocate} is called, $\langle \left|B_{l}^{(1)}\right|,\cdots,\left|B_{l}^{(M)}\right|\rangle$ is balanced;
    \item For any player $i$, the number of completed phases in stage 1 is at least $l_0=\lfloor \log_4\left(\frac{D}{C_2K\log MKT}\right)\rfloor$;
    \item The number of phases at stage 2 is at most $L=4+1.5\log(MK)=O\left(\log (MK)\right)$.
\end{enumerate}
\end{fact}
\begin{proof}
    1. Let $S^{+}=\{i:n_l^{(i)} \geq \bar{n}\}$ and $S^{-}=\{i:n_l^{(i)} < \bar{n}\}$.  At line 3 of server's \texttt{reallocate} code, server receives $\sum_{i\in S^{+}}\left(n_l^{(i)}-\bar{n}\right)$ arms. At line 5,  $\sum_{i\in S^{-}}(\bar{n}-n_l^{(i)})$ arms are removed. So at line 6 $|A_{temp}|=\sum_{i\in S^{+}}(n_l^{(i)}-\bar{n})-\sum_{i\in S^{-}}(\bar{n}-n_l^{(i)})=\sum_{i\in [M]}n_l^{(i)}-M\lfloor\sum_{i\in [M]}n_l^{(i)}/M\rfloor<M$.
    
    2. Let $n_l^{(i)}=\left|B_{l}^{(i)}\right|$. If $\max_i n_l^{(i)} \leq2\min_i n_l^{(i)}$, the reallocation procedure will do nothing, and $\langle n_l^{(1)},...,n_l^{(M)} \rangle$ is by definition balanced. If $\max_i n_l^{(i)} > 2\min_i n_l^{(i)}$, then at the end of the reallocation procedure, every player has a new set of arms $B_{l}^{(i)}$ such that $\bar{n} \le \left|B_{l}^{(i)}\right| \le \bar{n}+1$. This implies that the number of arms is balanced, since when reallocation is called, $\sum_i n_l^{(i)}\geq M$.
    
    3. The length of the $l$-th phase at stage $1$ is at most $K\lceil 4^{l+3}\log(MKT)\rceil$. After $l$ phases,  the number of timesteps $D_l$ satisfies $ D_l < \frac{64}{3}K\left(4^{l}-1\right)\log(MKT) + Kl < \frac{67}{3}K \cdot 4^{l}\log(MKT)$. Set $C_2=67/3$. We can see that the number of phases at stage 1 is at least $l_0 =\lfloor \log_4\left(\frac{D}{C_2K\log MKT}\right)\rfloor$.
    
    4. Suppose that phase $l$ is completed. Since at least $4^{l+3}\log(MKT)$ pulls are made in phase $l$, we can show that $4^{l+3}\log(MKT)\le T$. On the same time, $4^{l_0}\ge \frac{T}{C_2MK^2\log (MKT)}$. Therefore the number of phases at stage 2 satisfies $$l-l_0 \le \lceil\log_4\left(C_2M^2K^2\right)\rceil \le 4+1.5\log(MK) = L = O\left(\log(MK)\right).$$
    %4. Totally we need at most $O(\log(\frac{MT}{\log(MKT)}))$ phases to finish $MT$ timesteps. As the number of phases before d is at least $l_0 =\lfloor \log_4\left(\frac{T/(MK)}{C_2K\log MKT}\right)\rfloor$, the number of phases after D is at most $O(\log(MK))$.
\end{proof}

Via a direct application of Hoeffding's inequality and union bound, we have the following lemma.
\begin{lemma}
\label{heoffding}

Let $l_D$ be the maximum number of phases for all agents at stage 1. Denote the average rewards computed by agent $i$ in phase $l$ of stage 1 be $\hat{\mu}_{i,l}(\cdot)$. With probability at least $ 1 - 2l_D/(MKT) $, for all phases $l \leq l_D$, any agent $ i $, any arm $ a \in [K] $, 

$$ |\hat{\mu}_{i, l}(a) - \mu(a)| \leq 2^{-l-1} $$
\end{lemma}

\begin{proof}
Observe that $ l_D = O \left( \log_4 \left( D / (\log MKT) \right) \right)  $ and $l_D \geq l_0$ by Fact 1.3. 

For any agent $i$, any phase $l \leq l_D$, denote the empirical mean for arm $a$ in phase $l$ by $ \hat{u}_{i, l}(a) $.  By a direct application of Hoeffding's bound and union bound, we can observe that 
for any fixed $i, l, a$, by Hoeffding's bound, 
\begin{align*}
\Pr\left[\left|\Hat{\mu}_{i, l}(a)-\mu(a)\right|>2^{-l-1}\right]\le 2\exp\left\{-\frac{1}{2}m_l\cdot 4^{-l-1}\right\}\le \frac{2}{(MKT)^2}.
\end{align*}

Take a union bound for agents $ i \in [M] $, arms $ a \in [K] $, and phases $ l \leq l_D $, the desired result is proved.
\end{proof}

\begin{lemma}
At the end of stage 1, the following holds with probability $1-2l_D/(MKT)$: 
\begin{enumerate}
    \item $\forall i\in [M]$, for all $l \leq l_D$, $1\in A_{l}^{(i)}$  (If $A_l^{(i)}$ exists);
    \item $\forall i\in [M]$, for all $l \leq l_D$, $\forall a\in A_{l}^{(i)}$, $\mu(a)\geq \mu(1)-2^{-l+1}$  (If $A_l^{(i)}$ exists);
    \item $\exists i\in [M], 1\in B^{(i)}_{l_0+1}$;
    \item $\forall i\in [M], \forall a\in B^{(i)}_{l_0+1}$, $\mu(a)\geq \mu(1)-2^{-l_0+1}$;
    \item $N_M(i,D) \le C_5M\log(MKT)/\Delta_i^2$;
\end{enumerate}
Denote the event that the above holds by $\Lambda_2$.
\end{lemma}

\begin{proof}
These results are direct implications of lemma \ref{heoffding}.

1. Notice that 
    $$ \hat{\mu}_{i, l}(1) \geq \mu(1) - 2^{-l-1} \geq \mu(a) - 2^{-l-1} \geq \hat{\mu}_{i, l}(a) - 2^{-l}.$$
    with probability $1-2l_D/(MKT)$. Thus, arm 1 will never be eliminated throughout the first $l_D$ phases.
    
2. For any $a \in A_l^{(i)} $,
    $$\hat{\mu}_{i, l}(a) \geq \max_{k\in A_l^{(i)}}\hat{\mu}_{i, l}(k)-2^{-l} \geq \hat{\mu}_{i, l}(1) - 2^{-l}$$
    with probability $1-2l_D/(MKT)$, which means
    $$ \mu(a) \geq \hat{\mu}_{i, l}(a) - 2^{-l-1} \geq \hat{\mu}_{i, l}(1) - 2^{-l} - 2^{-l-1} \geq \mu(1) - 2^{-l+1}.$$

3. Let $A^{(i)}$ denote the remaining arms at the end of stage 1 for agent $i$. From 1 we know that $1 \in A^{(i)}$ for any $i \in [M]$ with probability $1-2l_D/(MKT)$. So at line 5 in protocol 1, arm 1 will be assign to agent $r_a$. 

4. Let $A^{(i)}$ denote the remaining arms at the end of stage 1 for agent $i$. From 2 we know that $\mu(a)\geq \mu(1)-2^{-l^{(i)}+1}$ for any arm in $\bigcup_{i \in[M]} A^{(i)}$ with probability $1-2l_D/(MKT)$. Here $l^{(i)}$ denotes the number of phases for agent $i$ at stage 1. Since $B_{l_{0}+1}=\bigcup_{i \in[M]} B_{l_0+1}^{(i)}$ is a subset of $\bigcup_{i \in[M]} A^{(i)}$ and $l_i \geq l_0$, we conclude that $\forall a \in B_{l_{0}+1}$, $\mu(a)\geq \mu(1)-2^{-l^{(i)}+1}\geq \mu(1)-2^{-l_0+1}$.

5. Let $l_a=\lceil \log_2 \Delta_a^{-1}\rceil$. Assume that $a\in A^{(i)}_{l_a}$. Since $\Delta_a>2^{-l_a}$, $a\notin A^{(i)}_{l_a+1}$. Therefore, the total number of times arm $i$ is pulled for any agent $i$ at stage 1 is
\begin{align*}
\sum_{l=1}^{l_a}2m_l\le \frac{C_5\log(MKT)}{\Delta_a^2}.
\end{align*}
Multiplying by $M$ proves the assertion.
\end{proof}

\begin{lemma}
\label{Hoeffding2}
Denote the following event by $\Lambda_1$: at stage 2, for any $l>l_0$, $a\in B_l$,
$$\left|\Hat{u}_l(a)-\mu(a)\right|<2^{-l-1}.$$
Then $\Pr\left[\Lambda_1\right]\ge 1-2L/(M^2KT)$.
\end{lemma}
\begin{proof}
Recall that for any $l>l_0$, $\Hat{u}_l(a)$ is the average of $m_l$ independent samples of the reward from arm $a$. Therefore, for fixed $l$ and $a$, by Hoeffding's inequality,
\begin{align*}
\Pr\left[\left|\Hat{u}_l(a)-\mu(a)\right|>2^{-l-1}\right]\le 2\exp\left\{-\frac{1}{2}m_l\cdot 4^{-l-1}\right\}\le \frac{2}{(MKT)^2}.
\end{align*}
A union bound for all $l_0+1 \leq l\leq l_0+L$ and all $a\in [K]$ proves this lemma.
\end{proof}

\begin{lemma}
Recall that $\Delta_i=\mu(1)-\mu(i)$. If $\Delta_i>0$, let $l_i=\lceil\log_2 \Delta_i^{-1}\rceil+1$. Then at stage 2, the total number of times such that arm $i$ is pulled is
$$N_M(i,T)-N_M(i,D) \le \frac{C_4\log(MKT)}{\Delta_i^2}+\eta(a),$$
where $C_4$ is a universal constant, and $\sum_{a\in [K]}\eta(a)\le M\log M$.
%$O\left(4^{l_i}\log MKT\right)$ 
\end{lemma}
\begin{proof}
    Suppose that $|\Tilde{B}_l|>M$. By Fact 1.2, at phase $l\ge l_0+1$ in stage 2, the sequence $\left\langle \left|{B}_{l}^{(1)}\right|,\cdots,\left|{B}_{l}^{(M)}\right| \right\rangle$ is balanced. Therefore, during phase $l$ at stage 2, the number of times an arm in $\Tilde{B}_{l}$ is pulled is at most $2m_l$.
    
    If $\left|\Tilde{B}_l\right|\le M$, an arm in $\Tilde{B}_{l}$ is also pulled at most $2m_l$ times, unless $\left|\Tilde{B}_l\right|\cdot m_l < M$. In that case, a phase only lasts for $1$ timestep. This is possible only when $m_l<M$, which requires $l<\lfloor \log_4 M\rfloor<\log M$. We denote the number of times $a$ is pulled in such phases by $\eta(a)$. %$\max\{2m_l, M\}$ times. The effect of the second term is only effective when $2m_l<M$, which requires $l<\lfloor \log_4 M\rfloor$. Therefore, the total effect of the second term in all phases is at most $$\sum_{l=1}^{\lfloor \log_4 M\rfloor}M\le M\log_4M.$$
    
    On the other hand, let $l_i=\lceil\log_2\frac{1}{\Delta_i}\rceil+1$. If $l_i\le l_0$, $\Lambda_2$ implies that $i\notin \Tilde{B}_{l_0+1}$. In that case, the number of times arm $i$ is pulled after timestep $D$ is $0$. Conditioned on event $\Lambda_1$, assume that $i\in \Tilde{B}_{l_i}$. Then
    \begin{align*}
        \Hat{u}_{l_i+1}&\ge \mu(1)-2^{-l_i-1} \ge \mu(i)+\Delta_i-2^{-l_i-1}\ge \Hat{u}_{l_i+1}(i) + 2^{-l_i}.
    \end{align*}
    Therefore, $\Lambda_1$ implies that $i\notin \Tilde{B}_{l_i+1}$. In this case, number of times arm $i$ is pulled is
    \begin{align*}
    N_M(i,T)-N_M(i,D) &\le \sum_{l=l_0+1}^{l_i+1}2m_l +\eta(a)\\
    &\le \frac{8}{3}4^{l_i+4}\log(MKT)+L+\eta(a) \\
    &\le \frac{2^{15}\log(MKT)}{3\Delta_i^2}+4+1.5\log(MKT)+\eta(a)\\
    &\le \frac{\left(2^{15}+13\right)\log(MKT)}{3\Delta_i^2}+\eta(a).
    \end{align*}
    We can also see that $\sum_{a\in [K]}\eta(a)\le M\log M$.
\end{proof}

\begin{lemma}
\label{lem:randominit}
Let $l\geq l_0+1$, and $n_l=\left|\Tilde{B}_{l}\right|$. With probability $1-2LK\delta$, either $n_l<21M\log\frac{1}{\delta}$ or no reallocation is performed before the start of the $l$-th phase at stage 2. %Suppose that $\mathbb{E}\left[n_l\right]>2C_3M\log\frac{1}{\delta}$. Then with probability $1-(l-l_0)\delta$, no reallocation is performed before the start of the $(l+1)$-th phase after $D$. 
\end{lemma}
\begin{proof}
Let $Y_{i,l}=\mathbb{I}\left[i\in \Tilde{B}_l\right]$, $X_{i,j}=\mathbb{I}\left[r_i=j\right]$, $i \in [K], j \in [M]$. Observe that $X_{i,j}$ and $Y_{i',l}$ are independent. This is because the elimination process between $\Tilde{B}_l$ and $\Tilde{B}_{l+1}$ uses exactly $m_l$ independent samples for each arm; therefore, the probability for remaining is independent of which player an arm is assigned to. Let $\Vec{Y}$ denote $\{Y_{i,l}, \forall i,l\}$. Since
\begin{align*}
\left|{B}_l^{(j)}\right| = \sum_{i=1}^{K}Y_{i,l}X_{i,j},\; \left|\Tilde{B}_l\right| = n_l = \sum_{i=1}^{K}Y_{i,l},
\end{align*}
by Chernoff's inequality,
\begin{align*}
\Pr\left[\left.\sum_{i=1}^{K}Y_{i,l}X_{i,j}>\frac{4}{3M}\sum_{i=1}^{K}Y_{i,l}\right|\Vec{Y}\right]&\le \exp\left\{-\frac{\sum_{i=1}^{K}Y_{i,l}}{21M}\right\},\\
\Pr\left[\left.\sum_{i=1}^{K}Y_{i,l}X_{i,j}<\frac{2}{3M}\sum_{i=1}^{K}Y_{i,l}\right|\Vec{Y}\right]&\le \exp\left\{-\frac{\sum_{i=1}^{K}Y_{i,l}}{18M}\right\}.
\end{align*}
Consequently, 
\begin{align*}
    &\Pr\left[\left.\frac{4\left|\Tilde{B}_l\right|}{3M}>\left|{B}_l^{(j)}\right|>\frac{2\left|\Tilde{B}_l\right|}{3M}\right|n_l>21M\log\frac{1}{\delta}\right]\\
    \ge& 1-\mathbb{E}_{\Vec{Y}}\left[\left. \exp\left({-\frac{n_l}{21M}}\right)+\exp\left({-\frac{n_l}{18M}}\right) \right|n_l>21M\log\frac{1}{\delta}\right]\\
    \ge & 1-2\delta.
\end{align*}
Note that $\frac{2}{3M}\left|\Tilde{B}_l\right|<\left|{B}_l^{(j)}\right|<\frac{4}{3M}\left|\Tilde{B}_l\right|$ for all $j\in[K]$ implies that $\left\langle\left|{B}_l^{(1)}\right|,\cdots,\left|{B}_l^{(M)}\right|\right\rangle$ is almost-uniform.
Therefore, if $n_l>21M\log\frac{1}{\delta}$, with probability $1-2(l-l_0)K\delta$, no reallocation will be performed before the $l$-th phase at stage 2.
\end{proof}

Now we are ready to prove Theorem \ref{thm:DEMAB}.
% <-- WYH Working; please don't modify
\setcounter{theorem}{0}
\begin{theorem}
\label{thm:DEMAB}
%Assuming that $T>MK$, the protocol given in the previous section gives $O\left(\sqrt{MKT\log T}\right)$ expected regret, and has $O\left(M\log(MK)\right)$ expected communication cost. 
The DEMAB protocol incurs $O\left(\sqrt{MTK\log T}\right)$ regret, $O\left(M\log\frac{MK}{\delta}\right)$ communication cost with probability $1-\delta$, and $O\left(M\log(MK)\right)$ communication cost in expectation.
%The regret that DEMAB protocol incurs is $\mathbb{E}\left[REG(T)\right]\le C\sqrt{MTK\log T}$, where $C$ is a universal constant. With probability $1-\delta$, the DEMAB protocol incurs $O\left(M\log\frac{MK}{\delta}\right)$ communication cost. With probability $1$, communication cost is $O\left(M\log(MK)+K\right)$
\end{theorem}
\begin{proof}
\textbf{Regret:} Condition on event $\Lambda_1 \land \Lambda_2$. Denote the arm pulled by agent $i$ on timestep $t$ by $a_{i,t}$. By definition,
\begin{align*}
REG(T)&=\sum_{t=1}^{T}\sum_{i=1}^{M}\left(\mu(1)-\mu(a_{i,t})\right)=\sum_{a=1}^{K}\Delta_a N_M(a,T)\\
&=\sum_{a=1}^{K}\Delta_a N_M(a,D)+\sum_{a=1}^{K}\Delta_a \left(N_M(a,T)-N_M(a,D)\right).
\end{align*}
By lemma 2, $\Lambda_2$ implies
\begin{align*}
\sum_{a=1}^{K}\Delta_a N_M(a,D) &= \sum_{a:\Delta_a>\epsilon}\Delta_a N_M(a,D) + \sum_{a:\Delta_a<\epsilon}\Delta_a N_M(a,D)\\
&\le \frac{C_5MK\log(MKT)}{\epsilon}+\epsilon MD.
\end{align*}
By choosing $\epsilon = \sqrt{C_5K\log(MKT)/D}$, we show that 
\begin{align}
\label{thm1:regb4D}
\sum_{a=1}^{K}\Delta_a N_M(a,D) &\le 2M\sqrt{C_5DK\log(MKT)}+M\log M \nonumber\\
&\le 2\sqrt{2C_5TM\log(MKT)}+M\log M.
\end{align}
By lemma 5, $\Lambda_2 \land \Lambda_1$ implies
\begin{align}
\sum_{a=1}^{K}\Delta_a \left(N_M(a,T)-N_M(a,D)\right) \le&  M\log M+\sum_{a:\Delta_a>\epsilon'}\Delta_a\left(\frac{C_4\log(MKT)}{\Delta_a^2}\right)\nonumber\\
&+\sum_{a:\Delta_a<\epsilon'}\Delta_a \left(N_M(a,T)-N_M(a,D)\right)\nonumber\\
\le&  M\log M+\frac{C_4K\log(MKT)}{\epsilon'}+\epsilon' MT\nonumber\\
=&  M\log M+2\sqrt{C_4KMT\log(MKT)}.\label{thm1:regafterD}
\end{align}
The last equation holds when $\epsilon' = \sqrt{C_4K\log(MKT)/(MT)}$. Since $TK\ge M\log M$ (which is assumed in the problem setting), $M\log M\le\sqrt{MTK\log (MTK)}$. Combining (\ref{thm1:regb4D}) and (\ref{thm1:regafterD}), we conclude that expected regret is
\begin{align*}
\mathbb{E}\left[REG(T)\right]&\le \left(2\sqrt{2C_5}+2\sqrt{C_4}+1\right)\sqrt{KTM\log(MKT)}+\left(\frac{L}{M^2KT}+\frac{1}{MKT}\right)MT\\
&\le C_6\sqrt{KTM\log(MKT)}\le C_6\sqrt{3KTM\log(T)}.
\end{align*}
Here $C_6=2\sqrt{2C_5}+2\sqrt{C_4}+1+7.5$ is a universal constant.

\textbf{Communication:} Total communication in stage 1 is $0$. We first consider the worst case communication cost in stage 2. Note that during a phase (either in distributed mode or in centralized mode), the communication cost is $O(M)$ if reallocation is not performed. The cost for reallocation at the start of phase $l$ is $O\left(\min\{n_l,n_{l'}-n_l\}\right)$, where $l'$ is the last phase where reallocation is performed. Summing this over all phases (at most $L=O(\log(MK))$) in stage 2, we conclude that total communication cost for reallocation is $O(ML+K)$.%and $O(M+\Delta_k)$ if reallocation is performed (here $\Delta_k$ refers to the number of eliminated arms after the last performed reallocation). Summing it over all phases in stage 2 (at most $L=O\left(\log (MK)\right)$), we conclude that total communication cost is $O\left(ML+K\right)$.

Define $l_1$ to be the first phase such that reallocation is performed. Then, our argument above shows that communication cost is $O(ML+n_{l_1})$. If $n_{l_1}>21M\log\frac{1}{\delta}$, the event in lemma~\ref{lem:randominit} will be violated for some $l$. The probability for that is at most $2L^2K\delta$. By resetting $\delta$, we can show that with probability $1-\delta$, $n_{l_1}<21M\log\left(\frac{2L^2K}{\delta}\right)$. Therefore, with probability $1-\delta$, total communication cost is
$$O\left(ML+21M\log\left(\frac{2L^2K}{\delta}\right)\right)=O\left(M\log\frac{MK}{\delta}\right).$$

%Define $l^*$ to be the last completed phase at stage 2 such that $n_{l^*+1}>21M\log\left(\frac{2LK}{\delta}\right)$. Then by lemma \ref{lem:randominit}, with probability $1-\delta$, no reallocation is performed before the start of phase $l^*+1$. In that case, the communication cost before phase $l^*+1$ is $O\left(ML\right)$.

%On the other hand, at the start of round $l^*+1$, less than $R=21M\log\left(\frac{2LK}{\delta}\right)$ arms remain. Therefore, the communication cost afterwards is $O\left(ML+M\log\frac{LK}{\delta}\right)$.

%Putting things together, we conclude that with probability $1-\delta$, the communication cost is
%$$O\left(M\log\frac{MK}{\delta}\right).$$
In particular, by choosing $\delta=1/K$, we can show that expected communication cost is $O\left(M\log(MK)\right).$
\end{proof}

\setcounter{theorem}{4}
\begin{theorem}
\label{thm:DEMAB_instance_dependent}
When $D=l_0=0$ in DEMAB, the protocol incurs near-optimal instance-dependent regret $ O(\sum_{k: \Delta_k > 0} \Delta_k^{-1} \log T + M\log M) $. With probability $ 1 - \delta $ the communication cost is $ O\left(M \log (T / \delta)\right) $. The expected communication cost is $ O\left(M \log T\right)$.
\end{theorem}

\begin{proof}
\textbf{Regret:} In the case $D=0$, we can show that the number of phases is at most $ L' := O(\log T) $ via a similar arguments to Fact 1.4. By lemma 2.4, we can show that $ N_M(i, T) \leq O\left(\log (MKT) / \Delta_i^2 + \eta(a)\right) $. Therefore, the total regret bound is

$$ \sum_{i: \Delta_i > 0} N_M(i, T) \Delta_i = O\left(\sum_{i: \Delta_i > 0} \frac{\log T}{\Delta_i} + M \log M \right)$$

\textbf{Communication:} By the proof of Theorem \ref{thm:DEMAB}, the worst-case communication cost of this protocol is $ O\left(M \log T + K\right) $. This is because we need $ O(M) $ communication at the end of each phase to perform elimination, and at most a total number of $ O(K) $ additional communication among all phases to perform reallocation.

Let $ l^*$ be the last complete phase such that $ n_{l^* + 1} > 21 M \log (K L' / \delta) $. By lemma \ref{lem:randominit}, no reallocation is needed before phase $ l^* + 1 $ with probability $ 1 - \delta $, so the total communication before phase $ l^* + 1 $ is $ O(ML') $. 

From the beginning of phase $ l^* + 1 $, the total communication in the following phases is at most $ O(ML' + M \log (K L' / \delta)) $.

Therefore, with probability $ 1 - \delta $, the communication cost is $ O\left(M \log T + M \log (K L' / \delta)\right) = O\left(M \log (T / \delta)\right) $. Let $ \delta = 1 / K $, the expected communication is $ O\left( M \log T + (ML' + K) / K\right) = O\left( M \log T\right)$.

\end{proof}

\paragraph{Finite precision:} We now show that only $O(\log T)$ bits are needed for each number sent in DEMAB. The integers sent in DEMAB are numbers in $\{0,\cdots,K\}$. Therefore $O(\log T)$ bits are sufficient for each integer\footnote{Recall that $T>K$.}. In the DEMAB protocol, the only real numbers that are transmitted are the average reward in a phase. In our proof, it is only required that the average of $m_l$ samples is $\frac{1}{4^{l+2}\log(MKT)}$ subgaussian. In fact, the average of $m_l$ samples is $\frac{1}{4^{l+3}\log(MKT)}$ subgaussian. Therefore, we can use randomized rounding for the average reward with precision $\epsilon=\frac{1}{4TM}$. Then, the rounding error for one real number has zero mean, and is 
$%\epsilon^{-2}\le
\frac{1}{M4^{l+3}\log(MKT)}$ subgaussian. When computing the average of $m_l$ samples at phase $l$, at most $M$ rounding error terms will contribute to it, whose sum is $\frac{1}{4^{l+3}\log(MKT)}$ subgaussian. Therefore, the concentration inequality in lemma \ref{Hoeffding2} and consequently our main theorem still holds. Apparently $\log_2\frac{1}{\epsilon}=O\left(\log (MT)\right)=O(\log T)$. Therefore, expected number of communicated bits is $O\left(M\log(MK)\log T\right)$.

\section{Removing Public Randomness}
\label{sec:publicrandom}
The DEMAB protocol makes use of a public random number generator. It can be viewed as a sequence of uniformly random bits written on a public blackboard that every agent can read, and reading from this sequence does not require communication. In practice, this can be approximated by using a pseudorandom number generator with a truly random seed. In this case, regardless of the amount of public random number used, the communication cost is $O(M)$, which is the cost of broadcasting a short random seed.
%The DEMAB protocol exploits a certain amount of shared randomness. We now discuss how to remove the usage of public randomness with little increase in communication cost.

However, we can also totally remove the usage of public random numbers. The role of shared randomness in communication complexity has already been investigated. It is known that shared randomness can be efficiently replaced by private randomness and additional communication, as stated by the Newman's Theorem~\cite{newman1991private}. In our case, the argument is slightly different: we are considering an online learning task instead of function evaluation. Also, in DEMAB, the communication cost itself depends on the public random bits. In particular, we show the following theorem.
%The role of shared randomness in communication complexity has already been investigated; it is known that we can efficiently ``privatize'' shared randomness, as stated by the Newman's Theorem~\cite{newman1991private}. In our case, the argument is slightly different: we are considering regret instead of evaluating a function. In particular, we show the following theorem.

\begin{theorem}
	There exists a protocol for distributed MAB that does not use public randomness with expected regret $O(\sqrt{MKT\log T})$. It has communication cost bounded by $O(M\log (MK)+K)$ (worst case), and expected communication cost $O(M\log (MK))$.
\end{theorem}
\begin{proof}
We make the following modifications to the original DEMAB protocol. Instead of using a public random bit string $s$~\footnote{which has $K\lceil\log_2M\rceil$ bits}, we predetermine $B$ strings $s_1$,...,$s_B$ (which can be hardcoded in advance), and randomly choose from them. That is, the server will generate a random number uniformly distributed in $[B]$, and broadcast it to everyone. The communication cost of doing so will be $M\lceil\log_2 B\rceil$. We now analyze how the choice of $s_1,...,s_B$ affects the performance of the protocol.

In terms of regret bound, observe that for any random string $ s $, the expected regret of any bandits instance $ X $ is always $ O(\sqrt{MKT \log T}) $. Therefore, regret bound will not be affected when public randomness is removed.

Now define $f(X,s)$ to be the expected communication cost of the DEMAB protocol using the public random string $s$ and interacting with the multi-armed bandit instance $X$. Our analysis for DEMAB tells us that $\exists c_1$, $\forall X$,
\begin{align*}
\mathbb{E}_r\left[f(X,s)\right]\leq c_1M \log (MK).
\end{align*}
Therefore, if we draw i.i.d. uniform bit strings $s_1,...,s_B$,
\begin{align*}
\mathbb{E}_{s_1,...,s_B}\left[\frac{1}{B}\sum_{i=1}^{B}f(X,s_i)\right]\leq c_1M \log (MK).
\end{align*}
We say that a set of bit strings $\{s_1,...,s_B\}$ is bad for a bandit $X$ if 
$$\frac{1}{B}\sum_{i=1}^{B}f(X,s_i)>2c_1M \log (MK).$$

We know that there exists $ c_2 > 0 $ such that $ 0 \leq f(X,s_i) \leq c_2 (M \log (MK) + K) $. Therefore, by Hoeffding's inequality,
\begin{align*}
\Pr_{s_1,...,s_B}\left[\frac{1}{B}\sum_{i=1}^{B}f(X,s_i)>2c_1M \log (MK)\right]\leq \exp\left\{-\frac{2Bc_1^2 M \log (MK)}{c_2(1 + K / (M \log (MK)))}\right\}.
\end{align*}
In other words, for fixed $X$, the probability that we will draw a bad set $\{s_1,...,s_B\}$ is exponentially small. Therefore, for a family of bandits with size $Q$, the probability for drawing a set of $s_1,...,s_B$ that is bad for some bandit is at most $Q\cdot \exp\left\{-\frac{2Bc_1^2 M \log (MK)}{c_2(1 + K / (M \log (MK)))}\right\}$. If we can show that this quantity is smaller than $1$, it would follow that there exists $\{s_1,\cdots,s_B\}$ such that it is not bad for any bandit in the family.

Now, we consider the following family $\mathcal{X}$ of bandits. For each arm, the expected reward could be $a/\Delta$, where $a\in\left\{0,1,...,\lfloor\Delta^{-1}\rfloor\right\}$. The reward distribution is Bernoulli. The size of this family is $Q\le\left(\Delta^{-1}+1\right)^{K}$. Now, consider any other bandit $X_1$. Without loss of generality, we can assume that $X_1$ is a Bernoulli bandit, and that the expectation of each arm is in $[1/4,3/4]$\footnote{For a general bandit instance $X_1$, when reward $r$ is received, we can generate a Bernoulli reward with expectation $r/2+1/4$ to replace it. The regret bound will increase by only a constant factor.}. Apparently we can find a bandit $X_2\in \mathcal{X}$ such that their expected rewards are $\Delta$-close in $\Vert \cdot\Vert_\infty$. As a result, the KL-divergence of each arm's reward in $X_1$ and $X_2$ is $O(\Delta^2)$. Let $H(X)=\{a_{1,1},r_{1,1},...,a_{T,M},r_{T,M}\}$ be the random history of the DEMAB interacting with the bandit instance $X$. Since communication cost is determined given $H(X)$, $\forall s$,
\begin{align*}
\left|f(X_1,s)-f(X_2,s)\right|&\le c_2\left(M\log(MK)+K\right)d_{TV}(H(X_1),H(X_2))\\
&=O\left((M\log(MK)+K)\cdot \sqrt{TM}\Delta\right).
\end{align*}
With $\Delta=K^{-1}(MT)^{-0.5}$, the right-hand-side is $O\left(M\log(MK)\right)$. Therefore, it suffices to consider the bandit family $\mathcal{X}$.
%Then, by applying information-theoretic arguments used in the proof for regret lower bounds~\cite{lattimore2018bandit}, we can show that for any policy, their expected rewards differ by at most $O\left(MT\sqrt{MT\Delta^2}\right)$. Note that a communication protocol is just a special form of a single-agent policy on a $MT$-length trial; so this bound applies just as well. Therefore, if $\Delta=(MT)^{-1.5}$, it suffices to consider bandits in $\mathcal{X}$. In this case, $Q=O\left((MT)^{1.5K}\right)$.%we can ignore the rest of the bandits ($\mathcal{X}$ would be dense enough). In that case, $Q=(MT)^{1.5K}$.

Therefore, we only need to guarantee that $\frac{BM\log (MK)}{1 + K / (M \log (MK))}>C'K\log MKT$, where $C'$ is a universal constant. This can be met by setting $B=\lceil 2C'K^2 \log(MKT)\rceil$. In this case, we can guarantee that there exists a set of bit strings $\{s_1,...,s_B\}$, such that for any bandit instance $X$, when choosing $s$ randomly from this set, the expectation of $f(X,s)$ is $O(M\log(MK))$.

The additional communication overhead for generating the random string (in bits) is
$$O\left(M\log B\right)=O\left(M\log K + M \log \log ( MKT) \right).$$
Therefore, under our usual assumption that $T>\max\{M,K\}$, the number of total communicated bits is bounded by $O\left(M \log K + M \log \log T \right)$. In our formulation, we may view $\log T$ bits as one packet. Therefore additional communication cost is $O\left(M\right)$. It follows that total expected communication cost is $O(M\log(MK))$.
%Under our assumption $T>M+K$, this can be simplified to $O\left(M\log T\right)$.
\end{proof}

\section{Proof for Theorem 2}
\label{sec:lowerboundproof}
\begin{proof}
First, we list two lemmas that will be used in our proof.
\begin{lemma}
\label{lem:ucb_upper_bound}
(Theorem 9.1~\cite{lattimore2018bandit}) For $K$-armed bandits, there is an algorithm with expected regret
$$REG(T)\leq 38\sqrt{KT}.$$
\end{lemma}

\begin{lemma}
(Theorem 15.2~\cite{lattimore2018bandit}) For $K$-armed bandits, we can prove a minimax regret lower bound of
$$REG(T)\geq \frac{1}{75}\sqrt{(K-1)T}.$$
\end{lemma}
The original lower bound is proved for Gaussian bandits, which doesn't fit exactly in our setting. we modified the proof to work for Bernoulli bandits, which results in a different constant.

We now prove the theorem's statement via a reduction from single agent bandit to multi-agent bandit. That is, we map communication protocols to single-agent algorithms in the following way. For simplicity, we consider protocols as $M$ blocks of code. In agent $i$'s block, each line could be a local computation, sending a message, or waiting for a message to receive.

Consider a communication protocol with communication cost $B(M)$. We denote $X_i$ ($i\in[M]$) to be the indicator function for agent $i$'s sending or receiving an integer or a real number throughout a run. $X_i$ is a random variable. Since expected communication cost is less than $M/c$,%Since we have $\sum_i X_i\leq B(M)\leq M/c$ with probability 1,
$$\sum_{i=1}^M\mathbb{E}X_i\leq M/c.$$
Now consider the $S$, the set of $M/2$ agents with smallest $\mathbb{E}X_i$. For any $i\in S$, $\mathbb{P}(X_i\geq 1) \leq \mathbb{E}X_i\leq 2/c$. That is, for any of these agents, the probability of either speaking to or hearing from someone is less than $2/c$. Suppose that agent $j$ is such an agent. Then, we can map the communication protocol to a single-agent algorithm by simulating agent $j$.

The simulation is as follows. Interacting with single agent bandit with time $T$, we run the code for agent $i$ in the protocol. When no communication is needed, we may proceed to the next line of agent $i$'s code. When this line of code sends a message or waits for a message, we terminate the code. In the rest of the timesteps, we run a single-agent optimal algorithm (the one used to realize lemma \ref{lem:ucb_upper_bound}).

Then, if agent $j$'s code has $\delta$ probability of involving in communication, and if agent $j$'s regret $REG_j(T)\leq A$, via this reduction, we can obtain an algorithm for single-agent MAB with expected regret
$$REG(T)\leq A+\delta \cdot 38\sqrt{KT}.$$
By lemma $2$, $REG(T)$ cannot have a regret upperbound better than $\sqrt{T(K-1)}/75$. Therefore $$A+\delta\cdot 38\sqrt{KT}\geq \sqrt{(K-1)T}/75.$$
If $38\delta < 1/75$, we can show that $A=\Omega\left(\sqrt{KT}\right)$. In our case, let $c=3000$ will suffice. Since we can show this for any agent in $S$, we can show that total regret is $\Omega\left(M\sqrt{KT}\right)$.
\end{proof}

\section{Omitted Details of DELB}
\label{sec:detail_DELB}
\textbf{Assignment Strategy:} At line 4, we assign pulls to agents in the following way. Let $p_l=\lceil\sum_x m_l(x)/M\rceil$ denote the average pulls each agent needs to perform. Our assignment starts from the arm with the largest $m_l(x)$ and agent 1. For arm $x_k$ and agent $i$, if agent $i$ has been assigned with $p_l$ pulls, we turn to agent $i+1$. If we have finished designating $m_l(x_k)$ pulls for arm $x_k$, we continue designating arm $x_{k+1}$. The assignment is finished until all pulls are scheduled to agents. Observe that at the start of each phase, each agent has the same $A_l$ as the server. Therefore, at line $3$ they obtain the same $\pi_l$, with the support size at most $d \log \log d$. In that case, the server only needs to send a index ($O(1)$ communication cost) over $\xi=48 d \log \log d$ arms, instead of a vector ($\Omega(d)$ communication cost), to identify an arm $x\in \operatorname{Supp}(\pi_l)$.

\section{Proof of Theorem 3}
\label{sec:proof3}
First, we consider some properties of the elimination based protocol for linear bandits.
\begin{fact}
Suppose $T_l$ denotes the total number of pulls in the $l$-th phase, then we have
	$$C_14^ld^2\log MT\leq T_l\leq \xi + C_14^ld^2\log MT,$$
where $\xi = 48d\log\log d$.
\end{fact}
\begin{proof}
For arm $x$ in the core set, we pull it $\lceil C_1 4^l d^2 \pi_l(x) \log{MT}\rceil$ times. So we can directly find that the total number of pulls in phase $l$ satisfies $$C_14^ld^2\log MT\leq \sum_{a} m_l(a)\leq 48d\log\log d + C_14^ld^2\log MT$$
\end{proof}

\begin{lemma}
\label{lem:linucbcon}
	In phase $l$, with probability $1-1/TM$, for any $x\in \mathcal{D}$,
	$$\left|\langle\hat{\theta}-\theta^*,x\rangle\right|\leq 2^{-l}.$$
\end{lemma}
\begin{proof}
	First, construct an $\epsilon$-covering of $\mathcal{D}$ with $\epsilon_l = 2^{-l-2}$. Denote the center of the covering as $X=\{\bar{x}_1,...,\bar{x}_Q\}$. Here $Q$ satisfies $Q\leq 3^d2^{d(l+2)}$. 
	
	Assume that $\hat{\theta}$ is calculated from linear regression on $x_1$, ... , $x_{t'}$. For fixed $x\in \mathcal{D}$, it is known that $\langle\hat{\theta}-\theta^*,x\rangle$ is subgaussian with variance proxy 
	\begin{align}
	   \label{equ:subgaussian}
	   \sum_{s=1}^{t'}\langle x, V_l^{-1}x_s\rangle^2=\Vert x\Vert^2_{V_l^{-1}}\le 2\Vert x\Vert^2_{V_l^{-1}}.
	\end{align}
	Therefore with probability $1-2\delta$,
	$$\left|\langle\hat{\theta}-\theta^*,x\rangle\right|\leq 2\sqrt{\Vert x\Vert^2_{V_l^{-1}}\log\frac{1}{\delta}}.$$
	Suppose $n_l$ pulls are made in phase $l$. In our case,
	\begin{align*}
	\Vert x\Vert^2_{V_l^{-1}}\leq \frac{g(\pi)}{n_l} \leq \frac{2}{4^lC_1d\log MT}.
	\end{align*}
	Therefore with probability $1-2\delta$,
	$$\left|\langle\hat{\theta}-\theta^*,x\rangle\right|\leq 2^{-l+1}\sqrt{\frac{2}{C_1 d\log MT}\log\frac{1}{\delta}}.$$
	Choose $\delta = 1/(2TMQ)$. It can be shown that with $C_1=600$,
	\begin{align*}
	\frac{2\log(2MTQ)}{C_1 d\log MT}&\le \frac{\log 2+1+d\log 3+2d\log 2+d/2}{300d}\le \frac{1}{64}.
	\end{align*}
	%We can see that that there exists a $C_1$ such that with probability $1-1/(TM)$, for all $a\in X$
	Therefore with probability $1-1/(TM)$, for all $x\in X$
	$$\left|\langle\hat{\theta}-\theta^*,x\rangle\right|\leq 2^{-l-2}.$$
	Now, consider an arbitrary $x\in \mathcal{D}$. There exists $\bar{x}\in X$ such that $\Vert x-\bar{x}\Vert\leq 2^{-l-2}$. Therefore with probability $1-1/TM$, for any $x\in \mathcal{D}$,
	\begin{align*}
	\left|\langle\hat{\theta}-\theta^*,x\rangle\right|&\leq \left|\langle\hat{\theta}-\theta^*,\bar{x}\rangle\right|+\left|\langle\hat{\theta}-\theta^*,x-\bar{x}\rangle\right|\\
	&\leq 2^{-l-1}+\Vert \hat{\theta}-\hat{\theta^*}\Vert \cdot \Vert x-\bar{x}\Vert\\
	&\leq 2^{-l}.
	\end{align*}
	
\end{proof}

\begin{lemma}
	Let $x^*=\arg\max_{x\in D}\langle \theta^*,x\rangle$ be the optimal arm. Then with probability $1-\log (MT)/(TM)$, $x^*$ will not be eliminated until the protocol terminates.
\end{lemma}
\begin{proof}
If $x^*$ is eliminated at the end of round $l$, one of the following must happen: either (1) $\left|\langle\hat{\theta}-\theta^*,x^*\rangle\right|>2^{-l}$; or (2) there exists $x\neq x^*$, $\left|\langle\hat{\theta}-\theta^*,x\rangle\right|>2^{-l}$. Therefore the probability for $x^*$ to be eliminated at a particular round is less than $1-1/(TM)$. The total number of phases is at most $\log MT$. Hence a union bound proves the proposition.
\end{proof}

\begin{lemma}
	Suppose $\delta = 2\log (TM)/TM$, and $\Delta_x$ denotes the suboptimality gap of $x$, i.e. $\Delta_x = \left\langle\theta^*,x^*-x \right\rangle$. For suboptimal $x\in \mathcal{D}$, define $l_x=\inf\{l:8\cdot 2^{-l}\leq \Delta_x\}$. Then with probability $1-\delta$, for any suboptimal $x$, $x\not\in A_{l_x}$.
\end{lemma}
\begin{proof}
First, let us only consider the case where $x^*$ is not eliminated. That is, 
\begin{align*}
\Pr\left[\exists x\in \mathcal{D}: x\in A_{l_x}\right]&\leq \Pr\left[x^*\text{is  eliminated}\right]+\Pr\left[\exists x:x\in A_{l_x-1},x\in A_{l_x}|x^*\in A_{l_a}\right].
\end{align*}
Note that conditioned on $x^*\in A_{l_x}$, $\left\{x\in A_{l_x-1}\land x\in A_{l_x}\right\}$ implies that at phase $l_x-1$, either $\left|\langle\hat{\theta}-\theta^*,x\rangle\right|> 2^{-l_x+1}$ or $\left|\langle\hat{\theta}-\theta^*,x^*\rangle\right|> 2^{-l_x+1}$. Therefore the probability that there exists such $x$ is less than $\log (TM)/TM$. Hence, with probability $1-2\log (TM)/TM$, $x$ will be eliminated before phase $l_x$.
%$$\Pr\left[a\in A_{l_a}\right]\leq \Pr\left[a^*\in A_{l_a}\right]+\Pr\left[a\in A_{l_a}\right]$$
\end{proof}

We are now ready to prove our main result for DELB.

\setcounter{theorem}{2}
\begin{theorem}
\label{thm:linearelim}
DELB protocol has expected regret $O\left(d\sqrt{TM\log T}\right)$, and has communication cost $O\left((Md+d\log\log d)\log T\right)$.
\end{theorem}
\begin{proof}
\textbf{Regret:} We note that at the start of round $l$, the remaining arms have suboptimality gap at most $8\cdot 2^{-l}$. Suppose that the last finished phase is $L$. Therefore total regret is
\begin{align*}
REG(T)&\leq \sum_{l=1}^{L} C_14^ld^2\log MT\cdot 8\cdot 2^{-l}+\delta \cdot 2MT\\
&\leq C_3 2^Ld^2\log TM.
\end{align*}
Apparently $C_14^Ld^2\log TM\leq TM$. Therefore
\begin{align*}
REG(T)\leq \sqrt{C_3^24^Ld^4\log ^2TM}\leq C_7d\sqrt{TM\log TM}.
\end{align*}
Under our usual assumption that $T>M$, this can be simplified to $O(d\sqrt{TM\log T})$. Here $C_3$ and $C_7$ are some universal constants.

\textbf{Communication Cost:} 
Let $p_l=\sum_x m_l(x)/M$ denote the average pulls each agent needs to perform. Observe that for each arm, the number of agents that it is assigned to is at most $1+\lceil m_l(x)/p_l\rceil$ agents. Therefore, total communication for scheduling is at most
$$\sum_x \left(\lceil m_l(x)/p_l\rceil+1\right)\leq 2\xi+M=O(M+d\log\log d).$$
Similarly, total communication for reporting averages is the same. The cost for sending $\hat{\theta}$ is $Md$. Hence, communication cost per phase is $O(Md+d\log\log d)$. On the other hand, total number of phases is apparently $O(\log TM)$. Hence total communication is
$$O\left((Md+d\log\log d)\log TM\right)$$ 
Under the assumption that $T>M$, this can be simplified to $O\left((Md+d\log\log d)\log T\right)$.
\end{proof}

\paragraph{Finite precision: }We now discuss the number of bits needed in DELB. The integers in DELB are less than $\max\{T,\xi\}$ ($\xi=48d\log\log d$). Therefore, every integer can be encoded with $O(\log(dT))$ bits. It remains to be proven that transmitting each real number with logarithmic bits is sufficient. In the DELB protocol, two types of real numbers are transmitted: average of rewards, and entries of $\hat{\theta}$. To transmit real numbers with finite number of bits, we make the following modifications to the original protocol: 1. when transmitting average rewards at line $5$, use randomized rounding with precision $\epsilon_1=\frac{1}{M^2T}$; 2. after computing $\tilde{\theta}=V_l^{-1}X$ at line 7, let $\hat{\theta}$ be the entry-wise rounded vector of $\tilde{\theta}$ with $\epsilon_2=\frac{1}{MTd}$.%For the first type of real numbers, we transmit the average reward with randomized rounding of precision $\epsilon_1=\frac{1}{MT}$. When computing $\mu(a)$, at most $M$ terms of independent rounding error are involved, which sum up to a $\frac{1}{T}$-subgaussian random variable. 

It can be seen that we only need to prove that after the modifications, lemma~\ref{lem:linucbcon} still holds. In each phase, originally $\mu(x)$ is $\frac{1}{m_l(x)}$-subgaussian, but is only required to be $\frac{2}{m_l(x)}$-subgaussian for (\ref{equ:subgaussian}) to hold. After the modification, the contribution of rounding error to a $\mu(x)$ comes from at most $M$ independent terms, and is therefore subgaussian with variance proxy $\frac{1}{MT}\le \frac{1}{m_l(x)}$. Therefore, after the modifications, the computed $\mu(x)$ is $\frac{2}{m_l(x)}$-subgaussian; hence, (\ref{equ:subgaussian}) holds. It follows that with probability $1-1/(TM)$, for all $x\in X$, $\left|\langle\Tilde{\theta}-\theta^*,x\rangle\right|\leq 2^{-l-2}.$ Therefore, for any $a\in\mathcal{D}$,
$$\left|\langle \tilde{\theta}-\theta^*,x \rangle\right|\le 2^{-l-1}.$$
Combined with the fact that for any $x\in\mathcal{D}$,
$$\left|\langle \tilde{\theta}-\hat{\theta},x \rangle\right|\le \Vert \tilde{\theta}-\hat{\theta}\Vert \le \sqrt{d}\epsilon_2 \le \frac{1}{MT}\le 2^{-l-1},$$
we can prove that  with probability $1-1/(TM)$, for all $x\in X$, $$\left|\langle\Hat{\theta}-\theta^*,x\rangle\right|\leq 2^{-l-2}.$$
Therefore, after the modifications, the regret of the protocol is still $O\left(d\sqrt{TM\log T}\right)$. The amount of communicated bits is $$O\left(\left(Md+d\log\log d\right)\cdot \log T\cdot \log (dT)\right).$$

\section{Detailed Description of DisLinUCB}
\label{sec:detail_dislinucb}
\begin{algorithm}[ht]
    \caption{Distributed Linear UCB (DisLinUCB)}
    \label{alg:ucbs_linear}
    $D={T\log MT}/({dM})$, $\lambda=1$\\
    \For{Agent $i=1,...,M$}{
        Set $W_{syn,i}=0$, $U_{syn,i}=0$, $W_{new,i}=0$, $U_{new,i}=0$, $t_{last}=0$, $V_{last}=\lambda I$\\
    }
    \For{$t=1,...,T$}{
        \For{Agent $i=1,...,M$}{
        $\overline{V}_{t,i}=\lambda I+W_{syn,i}+W_{new,i}$, $\hat{\theta}_{t,i}=\overline{V}_{t,i}^{-1} \left(U_{syn,i}+U_{new,i}\right)$.\\
        Construct the confidence ellipsoid $C_{t,i}$ using $\overline{V}_{t,i}$ and $\hat{\theta}_{t,i}$.\\
        %Compute the confidence ellipsoid using available samples of agent $i$ based on $W_{syn,i}+W_{new,i}$ and $U_{syn,i}+U_{new,i}$. Denote his results as ${V}_{t,i}$, $\hat{\theta}_{t,i}$ and $\mathcal{C}_{t,i}$\\
		$(x_{t,i},\tilde{\theta}_{t,i})=\arg\max_{(x,\theta)\in \mathcal{D}_{t}\times \mathcal{C}_{t,i}} \langle x,\theta\rangle$\\
		Play $x_{t,i}$ and get the reward $y_{t,i}$.\\
		Update $W_{new,i}=W_{new,i} + x_{t,i}x_{t,i}^T$, $U_{new,i}= U_{new,i}+x_{t,i}y_{t,i}$.\\
		$V_{t,i}=\lambda I+W_{syn,i}+W_{new,i}$\\
		 \If{$\log\left({\det V_{t,i}}/{\det V_{last,i} }\right)\cdot (t-t_{last}) > D$}{
		 Send a synchronization signal to server to start a communication round.\\
		 %Communicate $W_{new,i}$ and $U_{new,i}$ via server.\\
		 %Receive $W_{syn,i}= W_{syn,i}+ \sum_{j=1}^{M} W_{new, j}$, $U_{syn,i}= U_{syn,i}+ \sum_{j=1}^{M} U_{new, j}$ from server.\\
		 %Set $W_{new,i}=0$, $U_{new,i}=0$, $t_{last}=t$, $V_{last}=\lambda I+W_{syn}$\\
		 }
		 \If{A communication round is started}{
            Send $W_{new,i}$ and $U_{new,i}$ to server\\
            Server computes $W_{syn}= W_{syn}+ \sum_{j=1}^{M} W_{new, j}$, $U_{syn}= U_{syn}+ \sum_{j=1}^{M} U_{new, j}$\\
            Receive $W_{syn}$, $U_{syn}$ from server\\
            Set $W_{new,i}=0$, $U_{new,i}=0$, $t_{last}=t$, $V_{last}=\lambda I+W_{syn}$
        }
        }
        
    }
\end{algorithm}

In DisLinUCB protocol, agent $i$ uses all samples available for him to maintain a confidence set $\mathcal{C}_{t,i} \subseteq \mathbb{R}^{d}$ for the parameter $\theta^*$ at each time step $t$. He chooses an optimistic estimate $\widetilde{\theta}_{t,i}=\operatorname{argmax}_{\theta \in \mathcal{C}_{t-1,i}}\left(\max _{x \in \mathcal{D}}\langle x, \theta\rangle\right)$ and then chooses action $x_{t,i}=\operatorname{argmax}_{x \in \mathcal{D}}\left\langle x, \widetilde{\theta}_{t,i}\right\rangle$, which maximizes the reward according to the estimate $\Tilde{\theta}_{t,i}$. We denote $\sum_{\tau} x_{\tau} x_{\tau}^{\top}$ and $\sum_{\tau} x_{\tau} y_{\tau}$ as $W$ and $U$ in our algorithm respectively. We use $W_{t,i}$ and $U_{t,i}$ to denote the sum calculated using available samples for agent $i$ at time step $t$. We construct the confidence set
 $\mathcal{C}_{t,i}$ using $\hat{\theta}_{t,i}$ and $\overline{V}_{t,i}$, which are constructed from $ W_{t, i} $ and $ U_{t, i} $:

\begin{equation}
\label{equ:OFUL_conf}
\mathcal{C}_{t,i}=\left\{\theta \in \mathbb{R}^{d} :\|\hat{\theta}_{t,i}-\theta\|_{\overline{V}_{t,i}} \leq \sqrt{2 \log \left(\frac{\operatorname{det}\left(\overline{V}_{t,i}\right)^{1 / 2} \operatorname{det}(\lambda I)^{-1 / 2}}{\delta}\right)}+\lambda^{1 / 2} \right\},
\end{equation}

where $\overline{V}_{t,i}=\lambda I+W_{t,i}$ and $\hat{\theta}_{t,i}=\left(\lambda I+W_{t,i}\right)^{-1} U_{t,i}$.

Our key observation is that the volume of the confidence ellipsoid depends on $\det(\overline{V}_t)$. If $\det(\overline{V}_{t,i})$ does not vary greatly, it will not influence the confidence guarantee even if the confidence ellipsoid is not updated. Therefore, we only need to synchronize when $\det(\overline{V}_{t,i})$ varies greatly. We refer to the timesteps between two synchronizations as an epoch. Since in the end, $\det(V_{last})$ is bounded, we can show that the number of epochs is limited. %We divide all steps into several epochs and only synchronize samples at the end of each epoch. As $\det(\overline{V}_T) $ is upper bounded, the number of those bad epochs, in which $\det(\overline{V}_t)$ varies greatly, is limited. %This inspires our analysis. Meanwhile, in our protocol, we also use $\det(\overline{V}_t)$ as part of our synchronization condition.

\section{Proof of Theorem 4}
\label{sec:proof4}
First of all, we state lemmas that will be used in our proof.

\begin{lemma}
\label{lem:thm2}
For any $\delta>0$, with probability $1-M\delta$, $\theta^{*}$ always lies in the constructed $\mathcal{C}_{t,i}$ for all $t$ and all $i$.
\end{lemma}

\begin{proof}
Using Theorem 2 in \cite{abbasi2011improved} and union bound over all agents, we can prove the lemma.
\end{proof}

For any positive definite matrix $ V_0 \in \mathbb{R}^{d \times d} $, any vector $ x \in \mathbb{R}^d $, define the norm of $ x $ w.r.t. $ V_0 $ as $ \left\|x\right\|_{V_0} := \sqrt{x^T V_0 x} $.

\begin{lemma}
\label{lem:lem11}      
(Lemma 11 in \cite{abbasi2011improved}) Let $\left\{X_{t}\right\}_{t=1}^{\infty}$ be a sequence in $\mathbb{R}^{d}$, $V$ is a $d \times d$ positive definite matrix and define $\overline{V}_{t}=V+\sum_{s=1}^{t} X_{s} X_{s}^{\top}$. Then we have that $$\log \left(\frac{\operatorname{det}\left(\overline{V}_{n}\right)}{\operatorname{det}(V)}\right) \leq \sum_{t=1}^{n}\left\|X_{t}\right\|^{2}_{\overline{V}_{t-1}^{-1}}.$$
Further, if $\left\|X_{t}\right\|_{2} \leq L$ for all $t$, then $$\sum_{t=1}^{n} \min \left\{1,\left\|X_{t}\right\|_{\overline{V}_{t-1}^{-1}}^{2}\right\} \leq 2\left(\log \operatorname{det}\left(\overline{V}_{n}\right)-\log \operatorname{det} V\right) \leq 2\left(d \log \left(\left(\operatorname{trace}(V)+n L^{2}\right) / d\right)-\log \operatorname{det} V\right).$$
%, and finally, if $\lambda_{\min }(V) \geq \max \left(1, L^{2}\right)$ then $$
%\sum_{t=1}^{n}\left\|X_{t}\right\|_{\overline{V}_{t-1}^{-1}}^{2} \leq 2 \log \frac{\operatorname{det}\left(\overline{V}_{n}\right)}{\operatorname{det}(V)}
%$$
\end{lemma}

Using Lemma \ref{lem:thm2}, we can bound single step pseudo-regret $r_{t,i}$. 
\begin{lemma}
        \label{lem:reg_linear_ucb}
		With probability $1-M\delta$, single step pseudo-regret $r_{t,i}=\langle\theta^*, x^*-x_{t,i}\rangle$ is bounded by
		\begin{equation}
		    \label{equ:linear_OFUL_regret}
		    r_{t,i}\leq 2\left(\sqrt{2\log\left(\frac{\det(\bar{V}_{t,i})^{1/2}\det(\lambda I)^{-1/2}}{\delta}\right)}  +\lambda^{1/2}\right)\lVert x_{t,i}\rVert_{\hat{V}_{t,i}^{-1}}=O\left(\sqrt{d\log \frac{T}{\delta}}\right)\left\lVert x_{t,i}\right\rVert_{\bar{V}_{t,i}^{-1}}.
		\end{equation}
	\end{lemma}
	\begin{proof}
		Assuming $\theta^*\in \mathcal{C}_{t,i}$,
		\begin{align*}
		r_{t,i}&=\langle\theta^*, x^*\rangle -\langle \theta^*,x_{t,i}\rangle\\
		&\leq \langle\tilde{\theta}_{t,i}, x_{t,i}\rangle -\langle \theta^*,x_{t,i}\rangle\\
		&= \langle \tilde{\theta}_{t,i}-\theta^*, x_{t,i}\rangle \\
		&=\langle \tilde{\theta}_{t,i}-\hat{\theta}_{t,i}, x_{t,i}\rangle + \langle \hat{\theta}_{t,i}-\theta^*, x_{t,i}\rangle\\
		&\leq \left\lVert \tilde{\theta}_{t,i}-\hat{\theta}_{t,i}\right\rVert_{\bar{V}_{t,i}} \lVert x_{t,i}\rVert_{\bar{V}_{t,i}^{-1}}
		+\left\lVert\hat{\theta}_{t,i}-\theta^*\right\rVert_{\bar{V}_{t,i}} \lVert x_{t,i}\rVert_{\bar{V}_{t,i}^{-1}}\\
		&\leq 2\left(\sqrt{2\log\left(\frac{\det(\bar{V}_{t,i})^{1/2}\det(\lambda I)^{-1/2}}{\delta}\right)}  +\lambda^{1/2}\right)\left\lVert x_{t,i}\right\rVert_{\bar{V}_{t,i}^{-1}}\\
		&=O\left(\sqrt{d\log \frac{T}{\delta}}\right)\left\lVert x_{t,i}\right\rVert_{\bar{V}_{t,i}^{-1}}.
		\end{align*}
	\end{proof}
	
	Now we are ready to prove Theorem \ref{thm:dislinucb}.
	\begin{theorem}
	\label{thm:dislinucb}
		DisLinUCB protocol achieves a regret of $O\left(d\sqrt{MT}\log^{2}(T)\right)$ with $O\left(M^{1.5}d^3\right)$ communication cost.
	\end{theorem}

\begin{proof}
\textbf{Regret: } Set $\delta = 1/(M^2T)$, the expected regret caused by the failure of Eq. (\ref{equ:linear_OFUL_regret}) is at most $MT\cdot 1/(MT)=O(1)$, thus we mainly consider the case where Eq. (\ref{equ:linear_OFUL_regret}) holds.

In our protocol, there will be a number of epochs divided by communication rounds. We denote $V_{last}$ in epoch $p$ as $V_{p}$. Suppose that there are $P$ epochs, then $V_{P}$ will be the matrix with all samples included.

Observe that $\det V_0=\det (\lambda I)=\lambda^d$. $\det(V_{P})\leq \left(\frac{tr(V_{P})}{d}\right)^d\leq \left(\lambda+MT/d\right)^d$. Therefore
		$$\log \frac{\det(V_{P})}{\det (V_0)}\leq d\log \left(1+\frac{MT}{\lambda d}\right).$$
		Let $R:= \lceil d\log \left(1+\frac{MT}{\lambda d}\right) \rceil$. It follows that for all but $R$ epochs, we have
\begin{equation}
    \label{equ:good_epoch}
    1\leq \frac{\det V_{j}}{\det V_{j-1}}\leq 2.
\end{equation}
We call those satisfying Eq. \ref{equ:good_epoch} good epochs. In these  epochs, we can use the argument for theorem 4 in \cite{abbasi2011improved}. First, we imagine the $MT$ pulls are all made by one agent in a round-robin fashion (i.e. he takes $x_{1,1}$, $x_{1,2}$,..., $x_{1,M}$, $x_{2,1}$,..., $x_{T,M}$).  We use $\tilde{V}_{t,i}= \lambda I+\sum_{\{(p,q):(p<t) \lor (p=t\land q<i)\}} x_{p,q} x_{p,q}^{T}$ to denote the $\overline{V}_{t,i}$ this imaginary agent calculates when he gets to $x_{t,i}$. If $x_{t,i}$ is in one of those good epochs(say the $j$-th epoch), then we can see that
		$$1\leq \frac{\det \tilde{V}_{t,i}}{\det \bar{V}_{t,i}}\leq \frac{\det V_{j}}{\det V_{j-1}}\leq 2.$$
		Therefore
		\begin{align*}
		r_{t,i}&\leq O\left(\sqrt{d\log \frac{T}{\delta}}\right)\sqrt{x_{t,i}^T\bar{V}_{t,i}^{-1}x_{t,i}}\\
		&\leq O\left(\sqrt{d\log \frac{T}{\delta}}\right)\sqrt{x_{t,i}^T\tilde{V}_{t,i}^{-1}x_{t,i}\cdot \frac{\det \tilde{V}_{t,i} }{\det \bar{V}_{t,i}}}\\
		&\leq O\left(\sqrt{d\log \frac{T}{\delta}}\right)\sqrt{2x_{t,i}^T\tilde{V}_{t,i}^{-1}x_{t,i}}.
		\end{align*}
We can then use the argument for the single agent regret bound and prove regret in these good epochs.

We denote regret in all good epochs as $REG_{good}$. Suppose $\mathcal{B}_p$ means the set of $(t,i)$ pairs that belong to epoch $p$, and $P_{good}$ means the set of good epochs, using lemma \ref{lem:lem11}, we have 

\begin{align*}
    REG_{good}&=\sum_{t}\sum_{i}r_{t,i} \\
    &\leq \sqrt{MT \sum_{p\in P_{good}}\sum_{(t,i)\in \mathcal{B}_p} r_{t,i}^{2}}\\
    &\leq O\left(\sqrt{dMT\log (\frac{T}{\delta})\sum_{p\in P_{good}}\sum_{(t,i) \in \mathcal{B}_p} \min \left(\left\|x_{t,i}\right\|_{\tilde{V}_{t,i}^{-1}}^{2}, 1\right)} \right)\\
    &\leq O\left( \sqrt{dMT\log(\frac{T}{\delta}) \sum_{p\in P_{good}} \log \left(\frac{\operatorname{det}\left(V_{p}\right)}{\operatorname{det}\left(V_{p-1}\right)}\right)}\right)\\
    &\leq O\left( \sqrt{dMT\log(\frac{T}{\delta})  \log \left(\frac{\operatorname{det}\left(V_{P}\right)}{\operatorname{det}\left(V_{0}\right)}\right)}\right) \\
    & \leq O\left(d\sqrt{MT}\log(MT)\right).
\end{align*}

Now we focus on epochs that are not good. For each bad epoch, suppose at the start of the epoch we have $V_{last}$. Suppose that the epoch starts from time step $t_0$, and the length of the epoch is $n$. Then agent $i$ proceeds as $\overline{V}_{t_0,i},...,\overline{V}_{t_0+n,i}$. Our argument above tells us that regret in this epoch satisfies $$REG\leq2\left(\sqrt{d\log T/\delta}\right)\sum_{i=1}^{M}\sum_{t=t_0}^{n}\min\left(\left\lVert x_{t,i}\right\rVert_{\bar{V}_{t,i}^{-1}},1\right)
\leq O\left(\sqrt{d\log T/\delta}\right)\cdot \sum_{i=1}^M\sqrt{n\log \frac{\det V_{t_0+n,i}}{\det V_{last}}}.$$
Now, for all but 1 agent, $n\log \frac{\det V_{t_0+n,i}}{\det V_{last}}<D$. Therefore we can show that %(we assume $\alpha\leq 1$)
 %$$REG(n)\leq O\left(d\log T/\delta\right)\cdot MD^{1/(2\alpha)}.$$
 $$REG(n)\leq O\left(\sqrt{d\log T/\delta}\right)\cdot M\sqrt{D}.$$
Since $\det(V_{P})\leq \left(\lambda+MT/d\right)^d$, we know that the number of such epochs are rare. (Less than $R=O(d\log MT)$). Therefore the second part of the regret is
$$REG_{bad}\leq O\left(Md^{1.5}\log ^{1.5}MT\right)\cdot D^{1/2}.$$
If we choose 
$D=\left(\frac{T\log MT}{dM}\right),$
then
$REG(T)=O\left(d\sqrt{MT}\log^2(MT)\right).$ Since $T>M$, we have
$$REG(T)=O\left(d\sqrt{MT}\log^2(T)\right).$$

\textbf{Communication: }Let 
$\alpha=\left(\frac{D T}{R}\right)^{0.5}.$ 
Apparently there could be at most $\lceil T/\alpha \rceil$ such epochs that contains more than $\alpha$ time steps. If the $j$-th epoch contains less than $\alpha$ time steps, $\log \left(\frac{\det V_{j+1}}{\det V_j}\right)>\frac{D}{\alpha}$. Since
$$\sum_{j=0}^{P-1}\log \left(\frac{\det V_{j+1}}{\det V_j}\right)=\log\frac{\det V_P}{\det V_0}\le R,$$
There could be at most $\lceil\frac{R}{D/\alpha}\rceil=\lceil\frac{R\alpha}{D} \rceil$ epochs with less than $\alpha$ time steps. Therefore, the total number of epochs is at most
$$\lceil\frac{T}{\alpha}\rceil+\lceil\frac{R\alpha}{D}\rceil=O\left(\sqrt{\frac{TR}{D}}\right).$$
With our choice of $D$, the right-hand-side is $O\left(M^{0.5}d\right)$. Communication is only required at the end of each epoch, when each agent sends $O(d^2)$ numbers to the server, and then downloads $O(d^2)$ numbers. Therefore, in each epoch, communication cost is $O(Md^2)$. Hence, total communication cost is $O\left(M^{1.5}d^{3}\right).$%At each communication round, we need to pass $O\left(Md^2\right)$ numbers. Hence total communication cost is $O\left(M^{1.5}d^{3}\right).$
\end{proof}
\paragraph{Finite precision:} we now consider the number of bits transmitted in the DisLinUCB protocol. To that end, we make the following minor modification to DisLinUCB. First, when reward $y_{t,i}$ is observed, we replace it with a random integer in $\{\pm 1\}$ with expectation $y_{t,i}$. Second, after line $8$, after $x_{t,i}$ is played, we round each entry of $x_{t,i}$ with precision $\epsilon$, and use the rounded vector in the calculation in line $9$. In this case, each entry of $W_{new,i}$ and $U_{new,i}$ is a multiple of $\epsilon^2$. Therefore, transmitting them requires $O(d^2\log \epsilon^{-1})$ bits. The total communication complexity is then $O\left(M^{1.5}d^3\log \epsilon^{-1}\right)$ bits.

We now discuss how to choose $\epsilon$ such that regret is not effected. Define
\begin{align*}
    \overline{REG}\left(\mathcal{H}\right)&:=\sum_{i=1}^M\sum_{t=1}^T\max_{x\in\overline{\mathcal{D}}_t}\langle x-x_{t,i},\theta^*\rangle,\\
    REG\left(\mathcal{H}\right)&:=\sum_{i=1}^M\sum_{t=1}^T\max_{x\in\mathcal{D}_t}\langle x-x_{t,i},\theta^*\rangle.
\end{align*}
Here $\mathcal{H}$ is a shorthand for a history $(x_{1,1},y_{1,1},\cdots,x_{T,M},y_{T,M})$. $\overline{\mathcal{D}}$ refers to set of rounded actions. For every $x\in\mathcal{D}$, there exists $\bar{x}\in\overline{\mathcal{D}}$ such that $\Vert x-\bar{x}\Vert\le \sqrt{d}\epsilon$. Therefore
$$\left|\overline{REG}\left(\mathcal{H}\right)-REG\left(\mathcal{H}\right)\right|\le MT\sqrt{d}\epsilon.$$
On the other hand, let $\mathcal{H}$ be the (random) history of the DisLinUCB with rounding on action sets $\mathcal{D}_t$, while $\overline{\mathcal{H}}$ is the (random) history of the DisLinUCB with rounding running on action sets $\overline{\mathcal{D}}_t$. Then at each time step, the mapping from past history to the next action is the same. Therefore $KL(\mathcal{H},\overline{\mathcal{H}})=O(MTd\epsilon^2)$. It follows that
$$\left|\mathbb{E}\left[\overline{REG}\left(\mathcal{H}\right)\right]-\mathbb{E}\left[\overline{REG}\left(\overline{\mathcal{H}}\right)\right]\right|\le O\left(\sqrt{M^3T^3d}\epsilon\right).$$
When the action set is $\overline{\mathcal{D}}_t$, no rounding is needed, so the regret analysis for the DisLinUCB protocol without rounding directly follows. Therefore, $$\mathbb{E}\left[\overline{REG}\left(\overline{\mathcal{H}}\right)\right]=O\left(d\sqrt{MT}\log^2T\right).$$
By choosing $\epsilon=(MT)^{-1}$, we can guarantee $\mathbb{E}\left[REG\left(\mathcal{H}\right)\right]=O\left(d\sqrt{MT}\log^2T\right)$ for any action set. In this case, the total number of communicated bits is $O\left(M^{1.5}d^3\log T\right)$.

\section{DEMAB and DELB in P2P Communication Networks}
\label{sec:comparep2p}
\SetAlgorithmName{Procedure}

In this section, we will briefly discuss how to implement our protocols (i.e. DEMAB and DELB) in the P2P communication network considered in~\cite{korda2016distributed}. We show that our protocols can be adopted to P2P networks after little modification. The communication cost will remain the same, while regret bounds would only increase marginally. 

In P2P communication networks, an agent can receive information from at most one other agent at a time, which leads to an information delay for each agent. In order to cope with such delay, we need to extend the length of each communication stage from 1 time step to $ M $ time steps, so that agents can complete the communication in turn. Since there are at most $ O(\log (MK)) $ communication stages in DEMAB and $ O(\log T) $ communication stages in DELB, the extension of communication stages incurs at most $O( M^2 \log (MK))$ regret in DEMAB and $O(M^2 \log T)$ regret in DELB for $M$ agents. When the time horizon $ T $ is large (i.e. $ T > M^3\log M $), the additional term is dominated by $ O\left(\sqrt{MKT \log T}\right) $ and $ O\left(d\sqrt{MT\log T}\right)$. Another issue for P2P networks is that there is no longer a physical server in the networks. To solve this problem, we can designate agent 1 as the server: agent 1 will execute both the codes for the server and the codes for an agent.%: agent 1 runs two copies of codes simultaneously--as an individual agent, and as the server.

Specifically, by saying ``extending the length of the communication stage'', we mean that we can use Procedure~\ref{alg:server2agent_P2P} and~\ref{alg:agent2server_P2P} to realize communication subroutines used in our protocols in P2P networks: sending message to the server and receiving messages from the server.%use the P2P network to simulate the following two behaviors in $ M $ time steps: the server (now agent 1) sends a message to each agent (Procedure \ref{alg:server2agent_P2P}); each agent reports a message to the server (Procedure \ref{alg:agent2server_P2P}). 

\begin{algorithm}[th]
    \label{alg:server2agent_P2P}
    \caption{Server2Agent: Agent 1 sends message $m_i$ to agent $i$ in a P2P network.}
    \DontPrintSemicolon
    For the next $M-1$ time steps: \\
    \tcc{For agent 1:}
    \quad  Send $m_i$ to agent $i+1$ at the $i$-th step\\
    \quad  Pull an arbitrary arm at each time step\\
    \BlankLine
    \tcc{For agent $i (i > 1)$:}
    \quad  Receive $m_i$ from agent $1$ at the $(i-1)$-st step\\ 
    \quad  Pull an arbitrary arm at each time step\\  
\end{algorithm}

\begin{algorithm}[th]
    \label{alg:agent2server_P2P}
    \caption{Agent2Server: Agents $i (i>1)$ sends message $m_i$ to agent 1 in a P2P network.}
    \DontPrintSemicolon
    For the next $M-1$ time steps: \\
    \tcc{For agent 1:}
    \quad  Receive $m_{i+1}$ from agent $i+1$ at the $i$-th step\\
    \quad  Pulls an arbitrary arm at each time step\\
    \BlankLine
    \tcc{For agent $i (i > 1)$:}
    \quad  Send $m_i$ to agent $1$ at the $(i-1)$-st step\\ 
    \quad  Pull an arbitrary arm at each time step\\  
\end{algorithm}

\subsection{DEMAB in P2P Networks}

For distributed DEMAB in a P2P network, we can replace the communication stage in DEMAB (i.e. line 8, 12, 15 of Protocol \ref{alg:DEMAB}) by Procedure Server2Agent and Agent2Server. In this way, it costs $M$ time steps instead of a single time step to collect, aggregate, and boardcast information. We have the following theorem showing the efficacy of distributing DEMAB in a P2P network.

%The detailed DEMAB protocol in P2P networks is shown as Protocol \ref{alg:DEMAB_P2P}. To simulate each communication stage of the original protocol, we use $M$ rounds of one-to-one information passing.%In each communication stage, we replace it by a $M$-step one-by-one information passing: each agent forwards their local data to agent 1, agent 1 aggregates the data and forwards it to agent 2, agent 2 then forwards it to agent 3, ..., finally agent $M$ obtains the aggregated data.

\setcounter{theorem}{6}
\begin{theorem}
\label{thm:DEMAB_P2P}

The DEMAB protocol in P2P networks incurs regret $ O\left(\sqrt{MKT \log T} + M^2 \log (MK)\right) $, with expected communication cost $ O\left(M \log (MK)\right)$.
When $ T > M^3 \log M $, the regret bound of this protocol is near-optimal $ O\left(\sqrt{MKT \log T}\right) $.
\end{theorem}

\begin{proof}
\textbf{Regret:} We compare DEMAB protocol in P2P net with the original one (i.e. Protocol \ref{alg:DEMAB}). The burn-in stage (i.e. Stage 1) of both protocols are the same. For distributed elimination stage (i.e. Stage 2), the length of each phase in the new protocol is no shorter than that in Protocol \ref{alg:DEMAB}. Therefore, the number of phases after phase $l_0 + 1$ (included) in new protocol is no more than that in Protocol \ref{alg:DEMAB}, which is $ O(\log (MK)) $. In each phase starting from phase $l_0 + 1$, new protocol needs $ O(M) $ additional steps to complete the communication in this phase, incurring $ O(M) $ additional regret per agent. Therefore, this protocol incurs $ O(M^2) $ additional regret per phase starting from phase $l_0 + 1$. The total regret of this protocol is thereby $ O\left(\sqrt{MKT \log T} + M^2 \log (MK)\right) $.

\textbf{Communication:} We still consider only distributed elimination stage. There are three communication stages per phase in Protocol \ref{alg:DEMAB}: Line 8, 12, and 15. 

In line 8 of DEMAB protocol, the total communication is $O(M)$ since $n_{max}$ is boardcast from the server, and each agent sends $ n_l^{(i)} $ to the server. We can observe that the communication cost at corresponding place is also $O(M)$ by replacing the boardcast with Server2Agent. In line 12, the communication cost of both protocols is still the same due to the same reason. In line 15, the new protocol calls Agent2Server which runs for $M$ steps, while the agents report the rewards in the original protocol in a single step. The communication  cost is $O(M)$ for both protocols.

In summary, the communication cost of the new protocol is the same as that of Protocol \ref{alg:DEMAB}, which is $ O\left(M \log (MK)\right) $.

\end{proof}

\subsection{DELB in P2P Networks}

\iffalse
\begin{algorithm}[H]
	\DontPrintSemicolon
	\label{alg:DELB_P2P}
	%\begin{algorithmic}[5]
	$A_1=\mathcal{D}$, $C_1=600$\\
	\For{$l=1,2,3,...$}{
	    \tcc{All agents: Solve a G-optimal design problem}
	     Find distribution $\pi_l(\cdot)$ over $A_l$ such that:
		1. its support has size at most $\xi=48d\log\log d$; 2. $g(\pi)\leq 2d$\\
		\tcc{Agent 1: Assign pulls and summarize results}
		For the next $M-1$ steps:\\
		\quad Agent 1: Assign $m_l(x)=\lceil C_14^{l}d^2\pi_l(x)\log MT \rceil$ pulls for each arm $x\in \operatorname{Supp}(\pi_l)$ and wait for results. \\ %$m_l(x)=\lceil C_14^{l}d^2\pi_l(x)\ln MT \rceil$\\
		\quad Other agents: pull arms arbitrarily \\
		Agent 1: Receive rewards for each arm $x\in A_l$ reported by agents\\
		Agent 1: For each arm in the support of $\pi_l(\cdot)$, calculate the average reward $\mu(x)$\\

		Agent 1: Compute
		$X=\sum_{x\in \operatorname{Supp}(\pi_l)} m_l(x)\mu(x)x$, $V_l=\sum_{x\in \operatorname{Supp}(\pi_l)} m_l(x)xx^{\top}$, $\hat{\theta}=V_l^{-1}X$ \\
		
		Agent 1: Forward  $\hat{\theta}$ to agent 2.\\
		For any agent $i$ in the next $M-2$ steps: \\
		\quad He pulls arms arbitrarily at each time step; If $i >1$ and he received $\hat{\theta}$ at last step, he forwards it to agent $i+1$ at current step.
		    
		\tcc{All agents: Eliminate low-rewarding arms}
		Eliminate low rewarding arms:
		$A_{l+1}=\left\{x\in A_l:\max_{b\in A_l}\langle\hat{\theta},b-x\rangle\leq 2^{-l+1}\right\}$
	}
	%\end{algorithmic}
	\caption{Distributed Elimination for Linear Bandits (DELB) in P2P Networks}
\end{algorithm}
\fi

Very similar to the P2P version of DEMAB, we can also distribute DELB to P2P networks by replacing the communication stage of DELB by Server2Agent and Agent2Server. We have the following theorem for the P2P version of DELB.

\begin{theorem}
\label{thm:DELB_P2P}

The DELB protocol in a P2P network has regret $ O\left(d \sqrt{MT 
\log T} + M^2 \log T\right) $ with expected communication cost $ O\left((Md + d \log \log d) \log T\right) $.
When $ T >M^3 \log M $, the regret of DELB is a near-optimal regret $ O\left(d\sqrt{MT \log T}\right) $.
\end{theorem}

\begin{proof}
The proof is very similar to the proof of theorem \ref{thm:DEMAB_P2P}. Note that in the P2P version of DELB, there are $ O(\log T) $ communication stages in total, which incurs $ O(M^2 \log T) $ additional regret. The communication cost of the new protocol is the same as Protocol \ref{alg:elms_linear} for the same reason mentioned in the proof of theorem \ref{thm:DEMAB_P2P}.
\end{proof}

\end{document}